\documentclass[a4paper,fleqn]{cas-sc}

\usepackage{amsmath}
\definecolor{cite_color}{rgb}{0.0, 0.58, 0.71}
\definecolor{db}{rgb}{0.0, 0.2, 0.7}
\usepackage{natbib}
\setcitestyle{authoryear,open={},close={}} 
\usepackage{hyperref}
\hypersetup{colorlinks=true,citecolor=cite_color,linkcolor=cite_color}
\usepackage{amsthm}
\usepackage{subfigure}
\usepackage{float}
\usepackage{pifont}
\newcommand{\cmark}{\ding{51}} 
\newcommand{\xmark}{\ding{55}} 
\usepackage[labelfont=it,textfont=it]{caption}

\newtheorem{thm}{Theorem}
\newtheorem{lem}[thm]{Lemma}
\newtheorem{prop}[thm]{Proposition}

\newtheorem{defn}{Definition}

\newtheorem{rem}{Remark}

\def\tsc#1{\csdef{#1}{\textsc{\lowercase{#1}}\xspace}}
\tsc{WGM}
\tsc{QE}
\tsc{EP}
\tsc{PMS}
\tsc{BEC}
\tsc{DE}

\usepackage{graphicx}
\usepackage{subcaption}
\renewcommand{\figurename}{Fig.}

\makeatletter
\renewcommand*{\fnum@figure}[1]{\figurename~\thefigure.}
\makeatother

\begin{document}
\let\WriteBookmarks\relax
\let\printorcid\relax 
\setlength{\skip\footins}{24pt}
\def\floatpagepagefraction{1}
\def\textpagefraction{.001}
\shorttitle{}
\shortauthors{Sixu Li et~al.}

\title [mode = title]{Lateral String Stability for Vehicle Platoons: Formulation, Definition, and Analysis}   

\author[1]{\textcolor{black}{Sixu Li}}

\credit{ Conceptualization, Methodology, Writing – original draft, Data curation, Software, Writing – review \& editing}

\author[2]{\textcolor{black}{Swaroop Darbha}}
\cormark[1]

\ead{dswaroop@tamu.edu}

\credit{Conceptualization, Methodology, Writing – review \& editing, Supervision}

\author[1]{\textcolor{black}{Yang Zhou}}
\cormark[1]

\ead{yangzhou295@tamu.edu}

\credit{Conceptualization, Writing – review \& editing, Supervision}

\address[1]{Zachry Department of Civil $\&$ Environmental Engineering, Texas A$\&$M University, College Station, TX 77843, USA}

\address[2]{J. Mike Walker '66 Department of Mechanical Engineering, Texas A$\&$M University, College Station, TX 77843, USA}

\cortext[cor1]{Corresponding authors}

\maketitle
\begin{abstract}
Platooning of connected and automated vehicles (CAVs) provides significant benefits in terms of energy efficiency, traffic throughput, and, most critically, safety. These safety benefits depend on string stability, which dictates how disturbances propagate along a vehicle string. Although longitudinal string stability has been extensively examined, lateral string stability, which governs the propagation of path-tracking errors that can lead to unsafe deviations from the desired path, remains underexplored. Its importance is growing as autonomous vehicles increasingly depend on onboard sensing and map‑free navigation, where sensor occlusions and tight formations amplify safety risks.
This paper presents a framework for lateral string stability that focuses directly on safety‑critical, path‑relative tracking errors and enables consistent comparison across vehicles that follow the same planned path. The key element of the framework is an arc‑length (Eulerian) viewpoint, a departure from traditional analyses, that clarifies how tracking errors at a given point on the path propagate from one vehicle to the next. Building on this foundation, we propose the definition of $\mathcal{L}_2$ lateral string stability along with two control strategies: a feedback-feedforward strategy that relies solely on onboard sensing, and a novel learn‑from‑predecessor strategy that makes use of vehicle‑to‑vehicle (V2V) communication. Both strategies are analyzed for lateral string stability with respect to two error measures: tracking error vector and lateral (cross-track) error. Our results show that onboard sensing alone cannot guarantee attenuation of path-tracking errors, imposing a fundamental safety limitation, while V2V communication enables true error attenuation. The analysis further identifies structural controller requirements, showing that nonzero feedback on specific measurements is essential for guaranteeing stability. The theoretical results are validated via simulation.
\end{abstract}

\section{Introduction}
Disturbance propagation in traffic flows plays a crucial role in stability, safety, and efficiency. This behavior is characterized by \textit{string stability}, which indicates whether disturbances attenuate as they travel along a vehicle platoon. The idea was first introduced in early work on decentralized control by \citet{chu1974decentralized}, and was subsequently formalized by \citet{swaroop1994string,swaroop2002string} to study how spacing errors propagate in automated vehicle (AV) platoons within the California PATH program. Since then, longitudinal string stability has been widely investigated under various control architectures and information flow topologies (\citep{swaroop1999constant,seiler2004disturbance,zheng2015stability}), and further generalized to capture communication imperfections such as packet loss (\citep{vegamoor2021string,yang2025robust}), noise (\citep{ma2024selection}), and time delays (\citep{ma2025robust}). The scope of applications has broadened to include mixed traffic settings (\citep{jin2014dynamics,jin2016optimal}) and human-driven vehicles (\citep{tian2025physically}), while theoretical generalizations have proposed $\mathcal{L}_p$-based definitions (\citep{ploeg2013lp}) and accounted for external disturbances acting on each vehicle in the platoon (\citep{besselink2017string}). In addition, the interplay between string stability, traffic flow, and the emergence of oscillations has been examined (\citep{mattas2023relationship, montanino2021string,sun2020relationship}).

In contrast, \textit{lateral string stability}, which determines how steering-induced path-relative tracking errors propagate, has received far less attention, even though it is critical for the safe operation of AVs. As AVs move away from infrastructure-dependent solutions (e.g., using magnet-guided lanes; \citealp{fenton1976steering}) toward onboard perception and \textit{map-free navigation} (\citep{ort2018autonomous,renz2025simlingo}), the lateral control aspect of platooning becomes both more significant and more demanding.

A key motivation stems from \textit{sensor occlusion} in platoons: a vehicle in front can block the field of view of cameras, LiDARs, and other sensors on the vehicle behind (\citep{zhang2021safe,yu2019occlusion}). Such occlusion degrades perception and makes it difficult for following vehicles to detect obstacles, lane markings, and other hazards in a timely manner. An illustrative example is shown in Fig. \ref{fig: scenario}, where a platoon encounters an obstacle and must maneuver around it safely. When the lead vehicle detects the obstacle, it can plan a safe avoidance path, as illustrated in Fig. \ref{fig: scenario}(a). By contrast, due to sensor occlusion, the second vehicle may not detect the obstacle until the lead vehicle has moved sufficiently far ahead, at which point it may be too late to plan and track a similarly safe path, as illustrated in Fig. \ref{fig: scenario}(b). This delayed perception effect can further accumulate downstream as the platoon size increases.
\begin{figure}[h]
    \centering
    \subfigure[safe avoidance by the lead vehicle]{\includegraphics[width=0.3\textwidth]{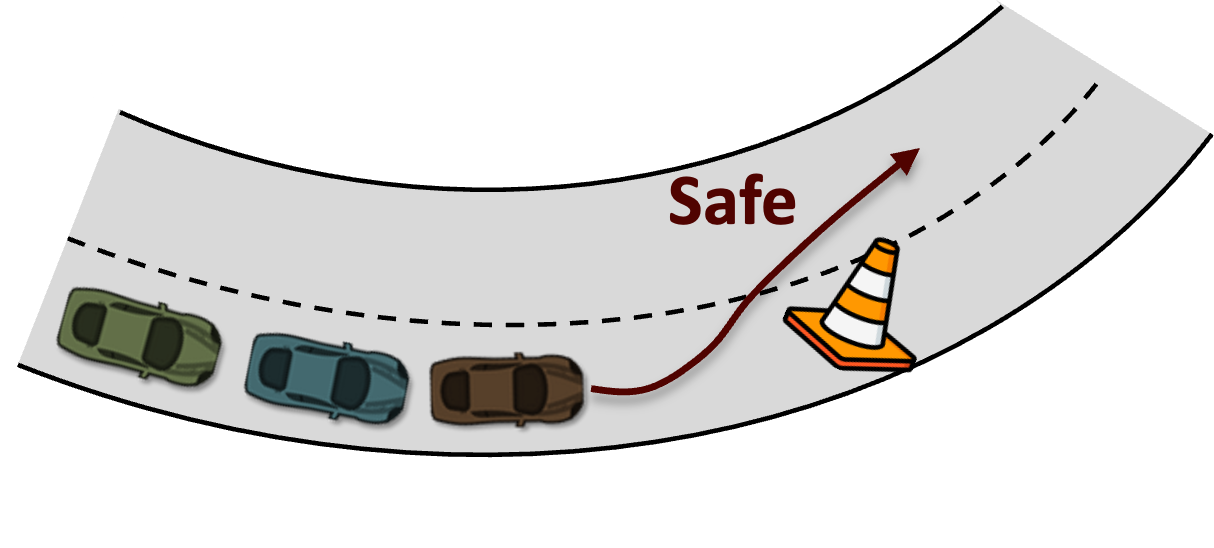}}
    \hspace{0.05\textwidth}
    \subfigure[unsafe response by the following vehicle]{\includegraphics[width=0.3\textwidth]{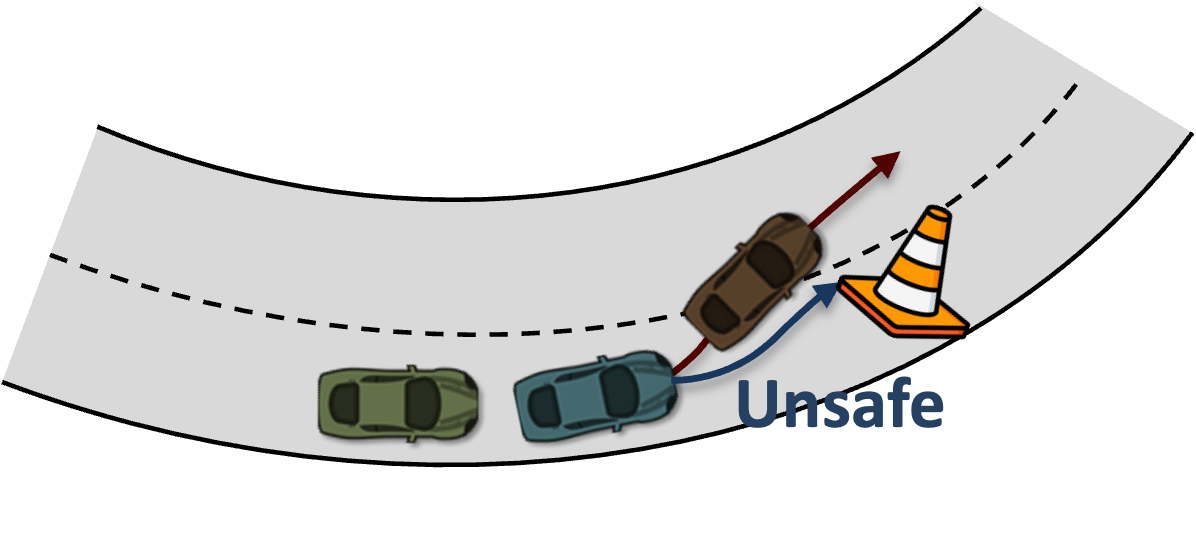}}
    \caption{Illustration of a sensor-occlusion scenario in a vehicle platoon}
    \label{fig: scenario}
\end{figure}

A practical solution to this issue is to let the lead vehicle plan a local path for the entire platoon (\citep{hassanain2020string,liu2020robust}), with following vehicles tracking sensor-recorded predecessor paths or receiving shared data via vehicle-to-vehicle (V2V) communication. However, under such schemes, the motion of one vehicle directly influences the next, thereby coupling the lateral motions of vehicles across the platoon. As a result, the relevant question is no longer only whether each vehicle can track its assigned path accurately in isolation, but also whether path-tracking errors attenuate or amplify as they propagate downstream. This motivates the study of lateral string stability in vehicle platoons. As a first step, a rigorous definition of lateral string stability is needed. Recent work has started to tackle this issue. \citet{alleleijn2014lateral} initially introduced the concept but did not supply a formal definition or detailed analysis. \citet{hassanain2020string} characterized lateral string stability using $\mathcal{H}_\infty$ norms of transfer functions, whereas \citet{liu2020robust} used $\mathcal{L}_\infty$ attenuation of tracking error vectors and showed that relying solely on predecessor information is inadequate. \citet{somisetty2024lateral} defined lateral string stability in terms of bounded tracking errors and illustrated the advantages of incorporating leader information. Nonetheless, these studies share key limitations:
\begin{itemize}
   \item A general definition of lateral string stability that is both rigorous and easy to apply in practice is still lacking.
    \item Existing studies typically examine particular signals (such as yaw rate or predecessor-relative errors). Yet, these quantities do not directly bound path-relative tracking errors, which are clearly critical for safety.
    \item Many previous works restrict attention to straight-line paths and do not systematically investigate the influence of different control designs or information topologies.
\end{itemize}

A central challenge is that lateral safety is fundamentally spatial: Since the road geometry and physical boundaries (e.g., path curvature, lane boundaries, and curbs) vary along the path, what matters is whether tracking errors attenuate as vehicles pass through the same location on the path. This corresponds to a position-based (Eulerian) viewpoint. However, vehicles move along the path asynchronously, so comparing errors at the same time instant typically compares different locations and curvatures, thereby conflating controller performance with geometric complexity. This motivates reparameterizing the dynamics with respect to path arc-length, so that tracking errors (and control actions) are represented as functions of distance along the path. Fig.~\ref{fig: reparameterization} illustrates the difference between comparing lateral (cross-track) errors $e_{lat}$ at a common time instant $t_0$ versus a common arc-length $l_0$, using a two-vehicle platoon as an example.
\begin{figure}[h]
    \centering
    \subfigure[time-based]{\includegraphics[width=0.55\textwidth]{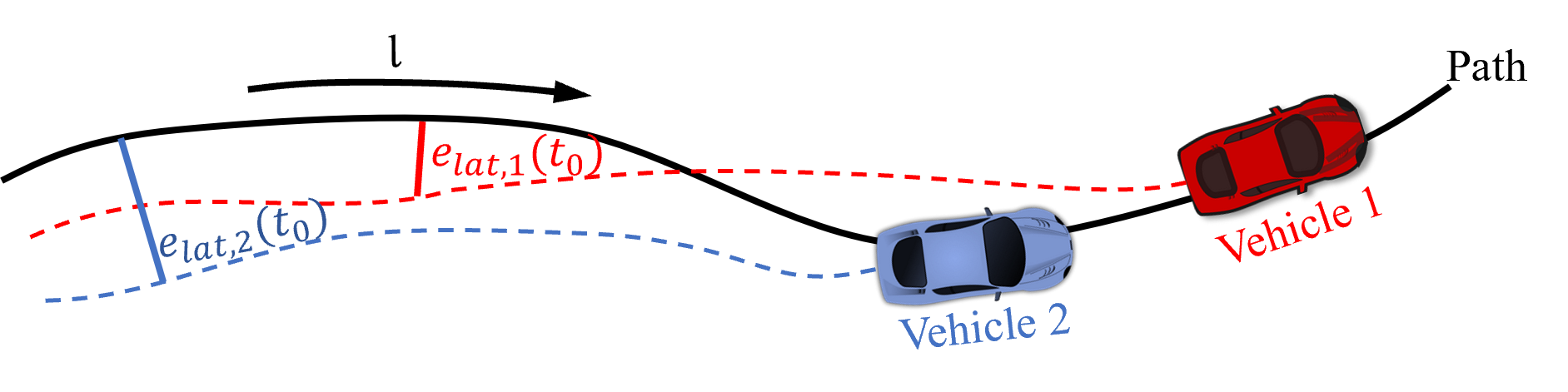}}
    \subfigure[arc-length-based]{\includegraphics[width=0.6\textwidth]{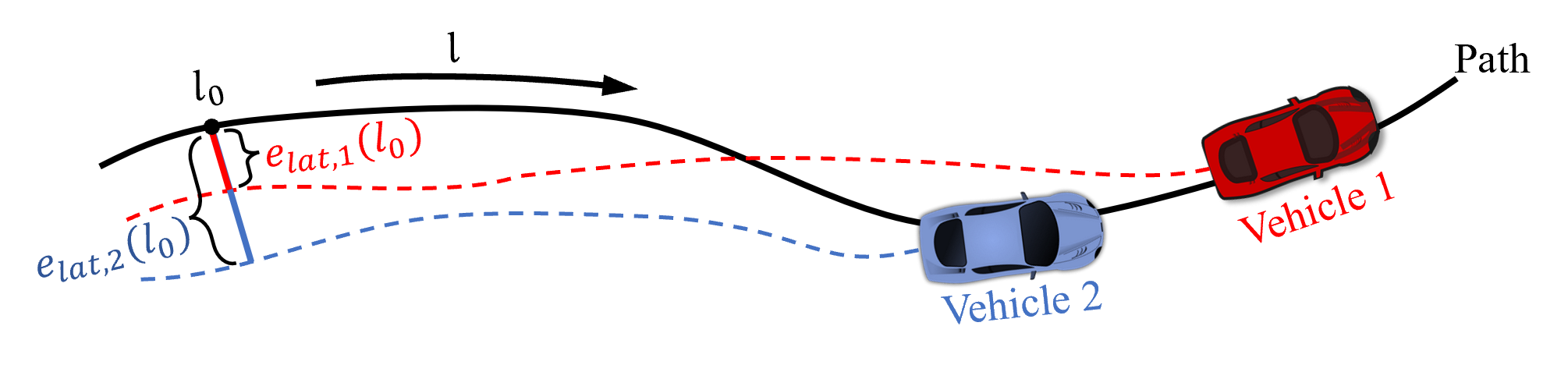}}
    \caption{Lateral error representations}
    \label{fig: reparameterization}
\end{figure}

This paper addresses the identified gaps through the following key contributions:
\begin{enumerate}
    \item A systematic formulation of the lateral control problem for vehicle platoons using arc-length reparameterization.
    \item A novel \textit{learn-from-predecessor} control strategy that improves tracking error attenuation.
    \item A definition of $\mathcal{L}_2$ lateral string stability.
    \item A general necessary and sufficient condition for $\mathcal{L}_2$ lateral string stability.
    \item A comprehensive analysis of different combinations of tracking schemes, control strategies, and error measures, including impossibility results and insights derived for controller design.
\end{enumerate}

The remainder of this paper is organized as follows. Section~\ref{sec3} formulates the lateral control problem for vehicle platoons. Section~\ref{sec4} proposes the formal definition of $\mathcal{L}_2$ lateral string stability and outlines a necessary and sufficient condition for it. Section~\ref{sec5} provides in-depth analyses of $\mathcal{L}_2$ lateral string stability. Section~\ref{sec6} validates the theoretical results through simulation. Section~\ref{sec7} concludes the paper.

\section{Formulation of the lateral control problem for vehicle platoons}
\label{sec3}

In this section, we formulate the lateral control problem for vehicle platoons. A schematic illustration is shown in Fig.~\ref{fig: platoon_setting}. We consider a finite platoon of $m$ vehicles, each tasked with tracking the same path by controlling its steering angle. The availability of path information may vary across vehicles.

\begin{figure}[htb!]
    \centering
    \includegraphics[width=0.6\linewidth]{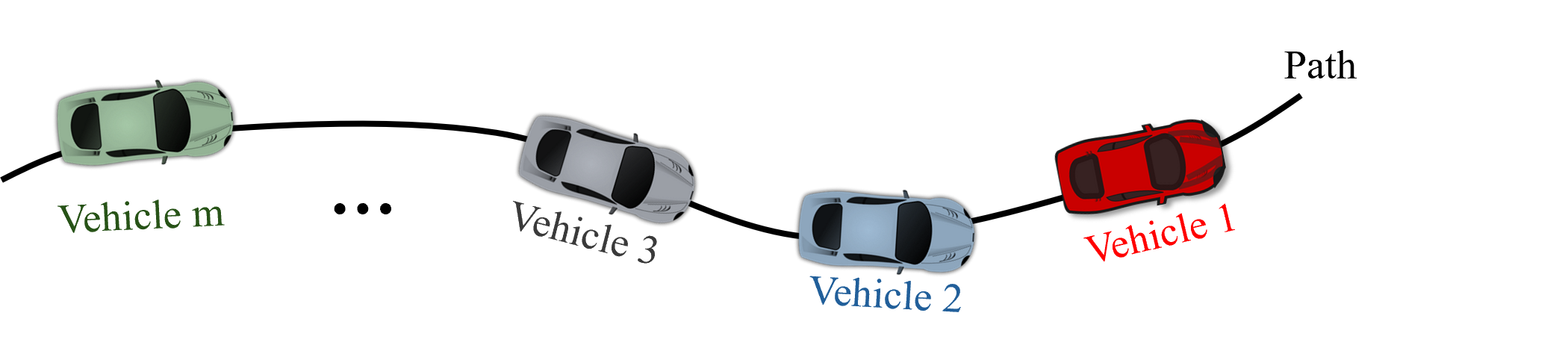}
    \caption{Schematic of the vehicle platoon lateral control problem}
    \label{fig: platoon_setting}
\end{figure}

To provide a systematic formulation, we introduce several key components:
\begin{itemize}
\item Subsection~\ref{sec2.2} introduces the tracking error dynamics of a vehicle.
    \item Subsection~\ref{sec3.1} derives an arc-length-based reparameterization of the tracking error model, enabling consistent comparison of errors across vehicles.
    \item Subsection~\ref{sec3.2} clarifies the distinction between the desired path and the reference path.
    \item Subsections~\ref{sec3.3} and~\ref{sec3.4} formulate two platoon lateral control modes based on information acquisition. The onboard-sensing mode employs feedback-feedforward control, whereas the V2V mode uses a proposed novel \textit{learn-from-predecessor} control strategy.
    \item Subsection~\ref{sec3.5} discusses the inherent robustness of platoon lateral control against time delays.
    \item Subsection~\ref{sec3.7} summarizes the control strategies in a comparative table.
\end{itemize}

\subsection{Tracking error dynamics model}
\label{sec2.2}

The dynamics of each vehicle in the platoon are described by the bicycle dynamics model, which is a widely used abstraction for planar vehicle dynamics that combines each axle into a single wheel (see, e.g., \citep{rajamani2011vehicle,liu2020lateral}). It represents the key aspects of lateral motion and yaw with adequate accuracy for control purposes. For path tracking controller design, it is useful to express the bicycle dynamics in terms of position and heading errors relative to the path. As shown in Fig.~\ref{fig: representation_heading_position_errors}, let $(X_v, Y_v)$ be the position of the vehicle's center of gravity (C.G.), and $(X_0, Y_0)$ its orthogonal projection onto the path. Define the lateral (cross-track) error $e_{lat}$ and heading error $\tilde{\theta} := \theta - \theta_R$, where $\theta_R$ is the desired heading angle of the path at $(X_0, Y_0)$. Noting that for simplicity, we use dot notation to denote time derivatives (e.g., $\dot{x}(t) = \frac{dx(t)}{dt}$), from kinematics (\citep{rajamani2011vehicle}):
\begin{equation}
\label{eq: d_theta_R}
\dot{\theta}_R = \frac{v_x}{R} = v_x \kappa,
\end{equation}
where $R$ is the instantaneous turning radius and $\kappa = \frac{1}{R}$ is the curvature of the path.

\begin{figure}[htb!]
    \centering
    \includegraphics[width=0.25\linewidth]{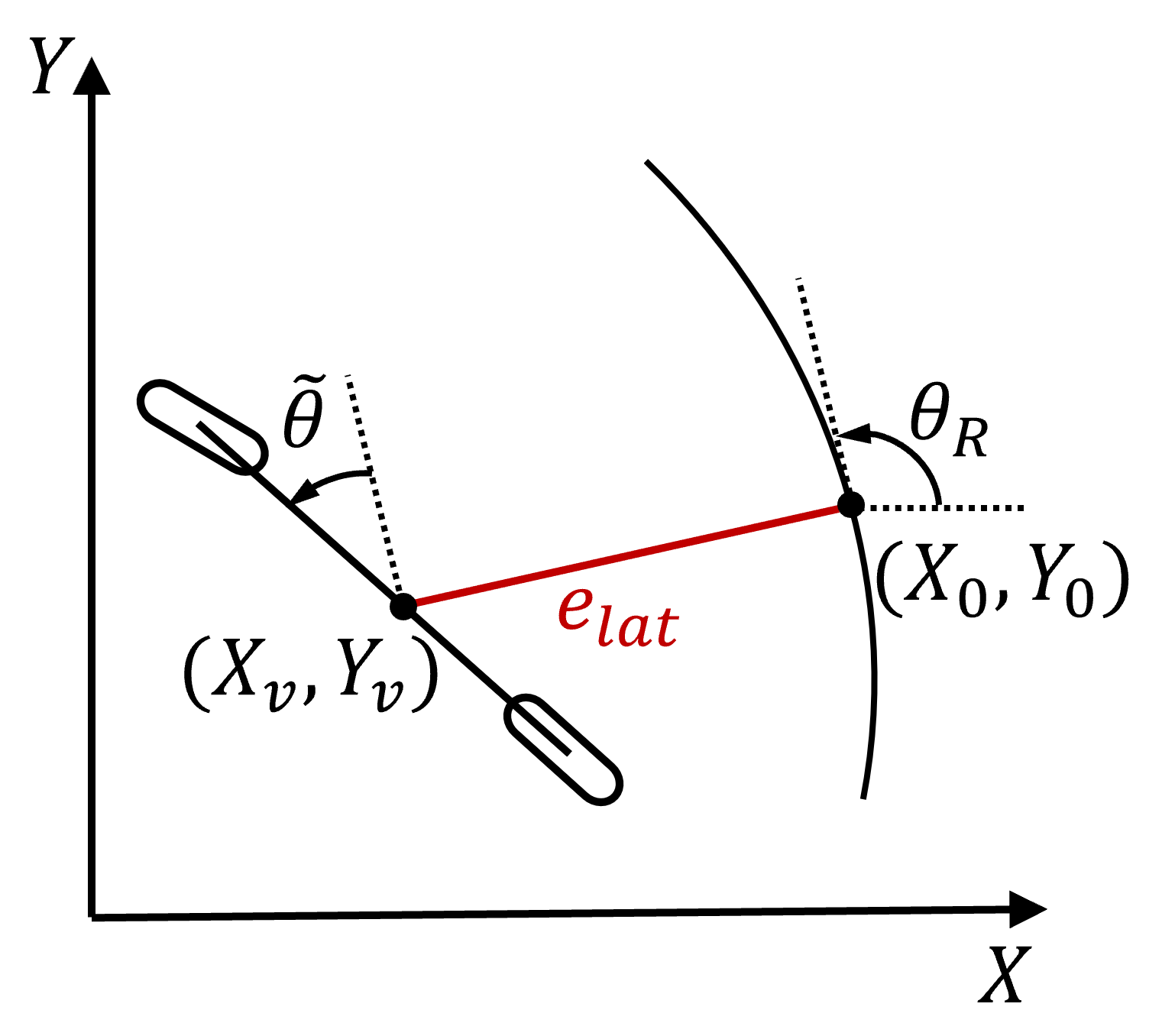}
    \caption{Heading and position error representation}
    \label{fig: representation_heading_position_errors}
\end{figure}

Under standard assumptions (\citep{liu2020lateral})—including constant longitudinal speed, small heading errors, negligible $\dot{R}$, large turning radius relative to $e_{lat}$, and omission of higher-order terms—the error dynamics derived are (\citep{rajamani2011vehicle,liu2020lateral}):

\begin{equation}
\label{eq: governing}
\mathbf{M} \ddot{\mathbf{e}}(t) + \mathbf{C} \dot{\mathbf{e}}(t) + \mathbf{L} \mathbf{e}(t) = \mathbf{B} u(t) - \mathbf{F} \kappa(t),
\end{equation}
where
\[
\mathbf{e} := \begin{bmatrix} e_{lat} \\ \tilde{ \theta} \end{bmatrix}, \quad
u := \delta_f, \quad
\mathbf{M} := \begin{bmatrix} m & 0 \\ 0 & I_z \end{bmatrix}, \quad
\mathbf{C} := \begin{bmatrix} \frac{C_f + C_r}{v_x} & \frac{a C_f - b C_r}{v_x} \\ \frac{a C_f - b C_r}{v_x} & \frac{a^2 C_f + b^2 C_r}{v_x} \end{bmatrix},   \quad
\mathbf{L} := \begin{bmatrix} 0 & -(C_f + C_r) \\ 0 & -(a C_f - b C_r) \end{bmatrix},
\]
\begin{align*}
\mathbf{B} := \begin{bmatrix} C_f \\ a C_f \end{bmatrix}, \quad
\mathbf{F} := \begin{bmatrix} mv_x^2 + a C_f - b C_r \\[1ex] a^2 C_f + b^2 C_r. \end{bmatrix}.
\end{align*}
Here, $\delta_{f}$ denotes the steering angle of the front wheels, while $a$ and $b$ are the distances from the C.G. to the front and rear axles, respectively. The coefficients $C_f$ and $C_r$ denote the cornering stiffness of the front and rear tires. The symbol $m$ denotes the vehicle mass, and $I_z$ is the moment of inertia about the vertical axis.

\subsection{Arc-length-based reparameterization of the tracking error model}
\label{sec3.1}

To analyze lateral string stability, we must compare tracking errors across vehicles. Since the road geometry and physical boundaries (e.g., path curvature $\kappa$, lane boundaries, and curbs) are inherently spatial properties, meaningful comparisons require evaluating errors at the same spatial location along the path. Thus, it is natural to reparameterize the tracking error model using the accumulated arc-length of the path, denoted by $l$.

This approach parallels reparameterization techniques used in \citep{besselink2017string} for longitudinal control. Fig.~\ref{fig: reparameterization} illustrates the difference between comparing lateral errors at a common time versus a common arc length, using a two-vehicle platoon as an example. The black solid line represents a curved path, with the black arrow indicating the direction in which the arc length $l$ increases along the path. The red and blue dashed lines represent the paths traveled by vehicle 1 and vehicle 2, respectively. In Fig.~\ref{fig: reparameterization}(a), the lateral errors of the two vehicles, $e_{lat,1}$ and $e_{lat,2}$, are shown at the same time instant $t_0$. Clearly, $e_{lat,1}(t_0)$ and $e_{lat,2}(t_0)$ are measured at different spatial locations along the path, due to the follower maintaining a longitudinal distance from its predecessor. In contrast, Fig.~\ref{fig: reparameterization}(b) shows $e_{lat,1}(l_0)$ and $e_{lat,2}(l_0)$, which are both evaluated with respect to the same spatial location $l_0$ on the path, making them suitable for direct comparison.

To derive the reparameterization, we project the vehicle's velocity onto the path tangent:
\begin{equation}
\label{eq: velocity proj}
    \frac{dl}{dt} = v_x \cos(\tilde{\theta}) - v_y \sin(\tilde{\theta}).
\end{equation}
Fig.~\ref{fig: Vel_projection} illustrates a schematic of this projection. Under the commonly adopted small heading error assumption in the literature (\citep{rajamani2011vehicle,liu2020lateral}), Eq. \eqref{eq: velocity proj} reduces to:
\begin{equation}
    \frac{dl}{dt} \approx v_x.
\end{equation}

\begin{figure}[htb!]
    \centering
    \includegraphics[width=0.27\linewidth]{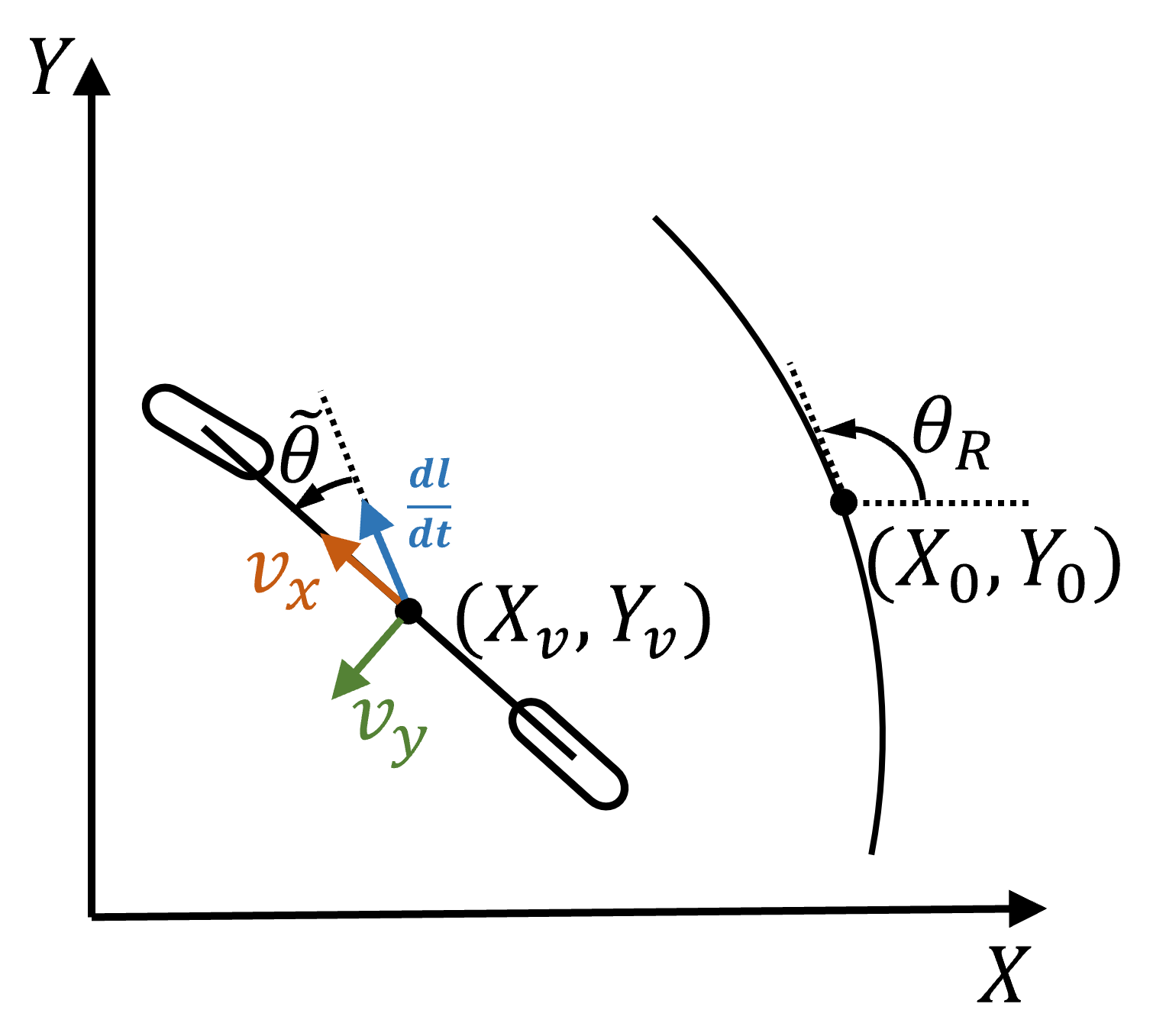}
    \caption{Projection of $v_x$, $v_y$ onto the path tangent}
    \label{fig: Vel_projection}
\end{figure}

Using the above equation and letting the prime notation denote arc-length derivatives ($x' = \tfrac{dx}{dl}$), we obtain:
\begin{align}
\label{eq: dot_e_reparam}
\dot{\mathbf{e}} &= \frac{d\mathbf{e}}{dl} \cdot \frac{dl}{dt} = \mathbf{e}' v_x,  \\
\label{eq: ddot_e_reparam}
\ddot{\mathbf{e}} &= \frac{d}{dt} \left( \frac{d\mathbf{e}}{dl} \cdot \frac{dl}{dt} \right)= \mathbf{e}'' v_x^2 + \mathbf{e}' \dot{v}_x \approx \mathbf{e}'' v_x^2,
\end{align}
where the last equality relies on the commonly adopted assumption of constant $v_x$ in the literature (\citep{rajamani2011vehicle,liu2020lateral}).

Substituting Eqs.~\eqref{eq: dot_e_reparam}–\eqref{eq: ddot_e_reparam} into Eq.~\eqref{eq: governing} yields the arc-length–based governing equation:
\begin{equation}
\label{arc governing}
v_x^2 \mathbf{M} \mathbf{e}''(l) + v_x \mathbf{C} \mathbf{e}'(l) + \mathbf{L} \mathbf{e}(l)
= \mathbf{B} u(l) - \mathbf{F} \kappa(l).
\end{equation}

While the reparameterization process in this subsection applies to any path, the notion of “the path” becomes ambiguous in a platoon context. The next subsection clarifies this by distinguishing between the concepts of the desired and reference path.

\subsection{Desired path versus reference path}
\label{sec3.2}
In the context of lateral control for vehicle platoons, it is important to distinguish between the concepts of the \textbf{\textit{desired path}} and the \textbf{\textit{reference path}}. While these two terms coincide in the single-vehicle setting, they may differ in a platoon scenario due to variations in information availability.

We define the \textbf{\textit{desired path}} as the ideal path that all vehicles in the platoon are expected to track. It is the same for the entire platoon and typically reflects the output of a path planning algorithm of the lead vehicle. The tracking error of the $i^\text{th}$ vehicle relative to the desired path is denoted by $\mathbf{e}_i^{des}$, and the curvature of the desired path is denoted by $\kappa^{des}$. Importantly, $\mathbf{e}_i^{des}$ represents the error we ultimately care about, and it is the error we seek to regulate or analyze in lateral string stability and performance assessment. The time-domain evolution of $\mathbf{e}_i^{des}$ can be described analogously to Eq.~\eqref{eq: governing} as:
\begin{align}
\label{eq: i_th_vechile_governing}
\mathbf{M} \ddot{\mathbf{e}}_i^{des}(t) + \mathbf{C} \dot{\mathbf{e}}_i^{des}(t) + \mathbf{L} \mathbf{e}_i^{des}(t) &= \mathbf{B} u_i(t) - \mathbf{F} \kappa^{des}(t)=\mathbf{B} u_i(t) - \tfrac{1}{v_x}\mathbf{F} \dot\theta^{des}(t),~~\forall i=1,2,\ldots,m,
\end{align}
where the second equality follows from Eq. \eqref{eq: d_theta_R}, and $\dot{\theta}^{des}$ denotes the desired yaw rate from the desired path. Assuming the small heading error condition holds for each vehicle relative to the desired path, we apply the reparameterization procedure from Subsection \ref{sec3.1} to Eqs. \eqref{eq: i_th_vechile_governing}–\eqref{eq: ith_veh_control}. This allows us to represent the error dynamics in terms of spatial arc length of the desired path, denoted $l_d$. Noting that $\dot\theta^{des}=\tfrac{d\theta^{des}}{dt}=\tfrac{d\theta^{des}}{dl_d}\tfrac{dl_d}{dt}=\left(\theta^{des}\right)'v_x$, we have
\begin{align}
\label{eq: i_th_governing_reparam}
v_x^2 \mathbf{M} \left( \mathbf{e}_i^{des}(l_d) \right)'' + v_x \mathbf{C} \left( \mathbf{e}_i^{des}(l_d) \right)' + \mathbf{L} \mathbf{e}_i^{des}(l_d)
= \mathbf{B} u_i(l_d) - \mathbf{F} \left(\theta^{des}(l_d)\right)',~~\forall i=1,2,\ldots,m,
\end{align} 

In contrast, the \textbf{\textit{reference path}} is the path used by each vehicle's controller for computing $u_i(l_d)$. This path may vary between vehicles depending on what information is locally available. For example, if the desired path is accessible to all vehicles through V2V communication, then each vehicle may use it as its reference path. However, if only the lead vehicle has access to the desired path, each following vehicle may instead record its predecessor's path using onboard sensors and use it as the reference. We will use $\mathbf{e}_i^{ref}$ to denote the tracking error of the $i^\text{th}$ vehicle relative to its reference path, and let $\kappa^{ref}_i$ and $\theta^{ref}_i$ denote the curvature and reference heading angle of that reference path, respectively. 

In the following two subsections, we differentiate between two information acquisition modes for the vehicles: onboard sensing and V2V communication. Each mode entails a different control strategy, as the available information—and therefore the reference path—differs between them.

\subsection{Onboard sensing mode using feedback-feedforward control}
\label{sec3.3}
In the onboard sensing mode, the lead vehicle directly tracks the desired path, while each following vehicle relies on its own sensors to record its predecessor's traveled path, expressed as X, Y coordinates and heading angle, as illustrated in Fig. \ref{fig: information flow_predecessor}. Specifically, a following vehicle can use onboard sensing modalities (such as stereo vision, LiDAR-based tracking, or camera–LiDAR fusion) to estimate the predecessor’s relative position and orientation, and then convert these into global X, Y, and heading by using its own GNSS-based pose estimate. 
\begin{figure}[h]
    \centering
    {\includegraphics[width=0.6\textwidth]{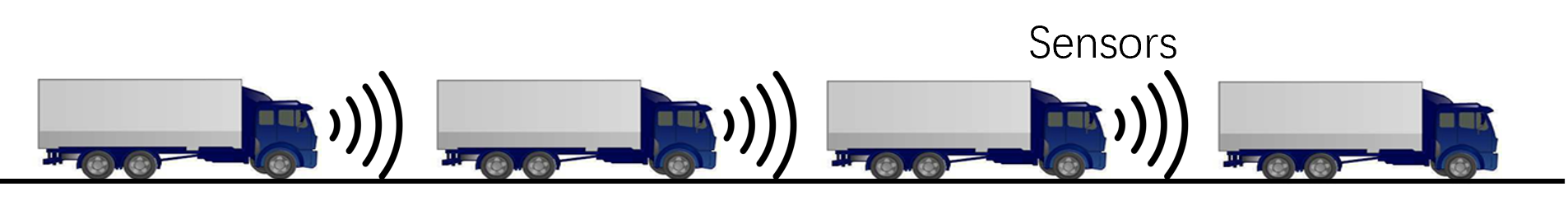}}
    \caption{Information flow under onboard sensing mode}
    \label{fig: information flow_predecessor}
\end{figure}

The lead vehicle uses the desired path as its reference path, while each following vehicle uses the path traveled by its immediate predecessor (recorded using onboard sensors) as its reference. We call this tracking scheme predecessor-tracking (PT). Fig.\ref{fig: ref_error_pred} illustrates the computation of reference lateral errors under PT. Heading errors are omitted for clarity. The black solid line represents the desired path, and the black arrow indicates the direction of increasing arc length $l_d$. The red and blue dashed lines represent the paths traveled by vehicles 1 and 2, respectively. Specifically, the reference error of the lead vehicle (vehicle 1) is measured relative to the desired path, whereas the reference error of the following vehicle (vehicle 2) is measured relative to the path traveled by its predecessor (vehicle 1). 
\begin{figure}[h]
    \centering
    {\includegraphics[width=0.7\textwidth]{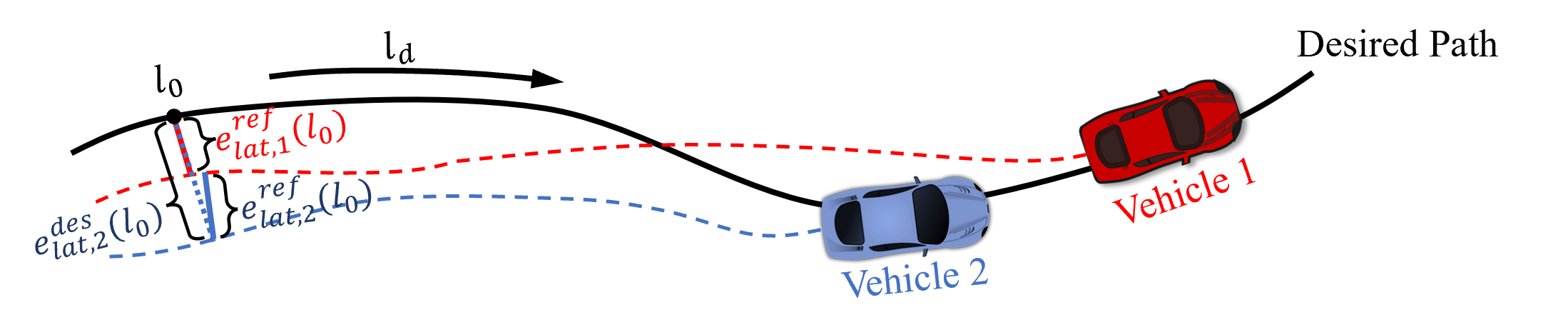}}
    \caption{Reference error computation under predecessor tracking}
    \label{fig: ref_error_pred}
\end{figure}

A \textbf{\textit{feedback–feedforward (FF)}} strategy, which is commonly adopted for single-vehicle lateral control and appreciated for its simplicity and practical effectiveness (\citep{peng1990vehicle,liu2020lateral,darbha2025robust}), is utilized for this mode. This strategy traces back to foundational work of \citep{fenton1976steering} and was subsequently advanced by \cite{guldner1999robust} in the area of robust lateral control. The FF strategy is chosen in this mode because each vehicle has access only to the reference path information.

The general form of the $i^\text{th}$ vehicle's FF control is given by:
\begin{align}
\label{eq: ith_veh_control}
u_i(t) &= -\mathbf{K_P} \mathbf{e}_i^{ref}(t) - \mathbf{K_D} \dot{\mathbf{e}}_i^{ref}(t) + \tfrac{k_{ff}}{v_x} \dot\theta^{ref}_i(t),~~\forall i=1,2,\ldots,m,
\end{align}
where the first two terms represent proportional–derivative (PD) feedback, and the last term is a feedforward component compensating for path curvature. The feedback gain vectors are defined as $\mathbf{K_P} = [k_{e_{lat}} \quad k_{\tilde{\theta}}]$ and $\mathbf{K_D} = [k_{\dot{e}_{lat}} \quad k_{\dot{\tilde{\theta}}}]$, while $k_{ff}$ is a scalar feedforward gain.

The feedback gains are chosen to ensure the stability of the closed-loop system when $\kappa \equiv 0$, i.e., for straight-line paths. Specifically, the characteristic polynomial of the closed-loop system derived from Eq.~\eqref{eq: governing} must be Hurwitz, ensuring that both $\mathbf{e}$ and $\dot{\mathbf{e}}$ converge to zero.
The feedforward term compensates for non-zero curvature. The choice of $k_{ff}$ depends on the design objective: \citet{peng1990vehicle} and \citet{liu2020lateral} selected $k_{ff}$ to match the vehicle’s steady-state yaw rate with the desired yaw rate $\dot{\theta}_R$, while \citet{rajamani2011vehicle} proposed tuning $k_{ff}$ to eliminate steady-state lateral error.

Similar to Subsection \ref{sec3.1}, we can reparameterize Eq. \eqref{eq: ith_veh_control} using $l_d$:
\begin{equation}
\label{eq: i_th_u_reparam}
u_i(l_d) = -\mathbf{K_P} \mathbf{e}_i^{ref}(l_d) - v_x \mathbf{K_D} \left( \mathbf{e}_i^{ref}(l_d) \right)'
+ k_{ff} \left(\theta^{ref}_i(l_d)\right)', ~~\forall i=1,2,\ldots,m.
\end{equation}

Note that in the controller implementation, $\mathbf{e}_i^{ref}$ and $(\theta_i^{ref})'$ are used for feedback and feedforward computation. However, for lateral safety, we are interested in the attenuation of $\mathbf{e}_i^{des}$. Hence, for analysis purposes, we need to relate $\mathbf{e}_i^{ref}$ and $\theta_i^{ref}$ to $\mathbf{e}^{des}_i$ and $\theta^{des}$. From Fig. \ref{fig: ref_error_pred}, observe that the blue solid line representing vehicle 2's lateral error relative to its reference path (the traveled path of vehicle 1) is perpendicular to the reference path, whereas the blue dashed line for vehicle 2's lateral error relative to the desired path is perpendicular to the desired path. Hence, under the assumption that vehicle 1 has a small heading error relative to its reference path (the desired path), which is consistent with the standard assumption outlined in Subsection \ref{sec2.2}, the path traveled by vehicle 1 is approximately locally parallel to the desired path. Consequently, the solid blue line and dashed blue line in Fig. \ref{fig: ref_error_pred} are approximately aligned in the same direction. Therefore, $e_{lat,2}^{des}(l_d)\approx e_{lat,1}^{ref}(l_d)+e_{lat,2}^{ref}(l_d)=e_{lat,1}^{des}(l_d)+e_{lat,2}^{ref}(l_d)$ and ${\tilde{\theta}}_2^{des}(l_d)\approx {\tilde{\theta}}_1^{ref}(l_d) + {\tilde{\theta}}_2^{ref}(l_d)={\tilde{\theta}}_1^{des}(l_d) + {\tilde{\theta}}_2^{ref}(l_d)$. Generalizing this relation to a multi-vehicle platoon, we have 
\begin{equation}
    \mathbf{e}_i^{ref}(l_d)=\begin{cases}
    \mathbf{e}_i^{des}(l_d)&\text{for }i=1,
    \\
    \\
    \mathbf{e}_i^{des}(l_d)- \mathbf{e}_{i-1}^{des}(l_d)&\text{for } i=2,3,\ldots,m.
    \end{cases}
\end{equation}
Moreover, by definition, the lead vehicle's reference heading angle corresponds to the heading angle specified by the desired path; whereas the reference heading angle of every following vehicle is the heading angle traveled by its predecessor, that is,
\begin{equation}
\theta^{ref}_i(l_d)=\begin{cases}
\theta^{des}(l_d)&\text{for }i=1,
\\
\\
\theta_{i-1}(l_d)=\theta^{des}(l_d)+\tilde{\theta}^{des}_{i-1}(l_d)=\theta^{des}(l_d)+[0~ 1]\mathbf{e}_{i-1}^{des}(l_d)&\text{for }i=2,3,\ldots,m. 
\end{cases}
\end{equation}
Hence, by Eq. \eqref{eq: i_th_u_reparam}, under PT, the formulation of the FF strategy for analysis purposes is given by
\begin{equation}
\label{eq: i_th_u_pred_track}
\small
u_i(l_d) =\begin{cases}
    -\mathbf{K_P} \mathbf{e}_i^{des}(l_d) - v_x \mathbf{K_D} \left( \mathbf{e}_i^{des}(l_d) \right)'
+ k_{ff} \left(\theta^{des}(l_d)\right)'&\text{for }i=1, \\
\\
-\mathbf{K_P} \left(\mathbf{e}_i^{des}(l_d)- \mathbf{e}_{i-1}^{des}(l_d)\right) - v_x \mathbf{K_D} \left(\mathbf{e}_i^{des}(l_d)- \mathbf{e}_{i-1}^{des}(l_d) \right)'
+ k_{ff}  \left(\theta^{des}(l_d)+[0~ 1]\mathbf{e}_{i-1}^{des}(l_d)\right)'&\text{for }i=2,\ldots,m.
\end{cases}
\end{equation}

\subsection{V2V mode using learn-from-predecessor control}
\label{sec3.4}
For the V2V mode, each following vehicle communicates with its predecessor, as shown in Fig. \ref{fig: information flow_desired}.
\begin{figure}[h]
    \centering
    {\includegraphics[width=0.6\textwidth]{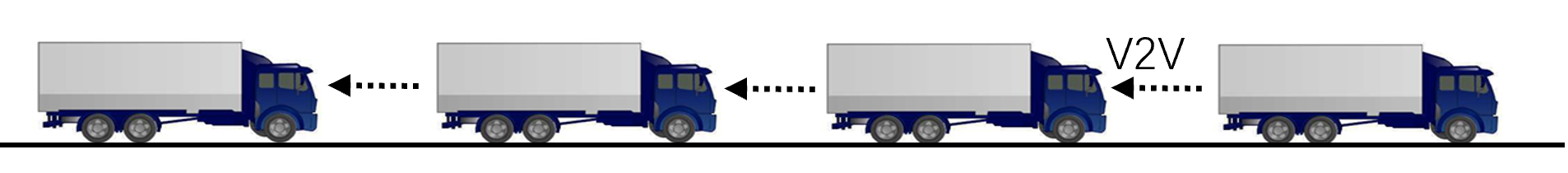}}
    \caption{Information flow under V2V mode}
    \label{fig: information flow_desired}
\end{figure}

V2V communication is utilized to transmit the desired path to all vehicles in the platoon, and every vehicle uses the desired path directly as its reference path. We call this tracking scheme desired-path tracking (DT). Here, for every $i$, we have $\mathbf{e}_i^{ref}(l_d)=\mathbf{e}_i^{des}(l_d)$, $\left( \mathbf{e}_i^{ref}(l_d) \right)'=\left( \mathbf{e}_i^{des}(l_d) \right)'$, and $\left(\theta^{ref}_i(l_d)\right)'=\left(\theta^{des}(l_d)\right)'$. Fig. \ref{fig: ref_error_desire} illustrates the reference lateral error computation under DT for a two-vehicle platoon. Since both vehicles use the desired path as their reference, their reference errors are both computed relative to the desired path.
\begin{figure}[h]
    \centering
    {\includegraphics[width=0.7\textwidth]{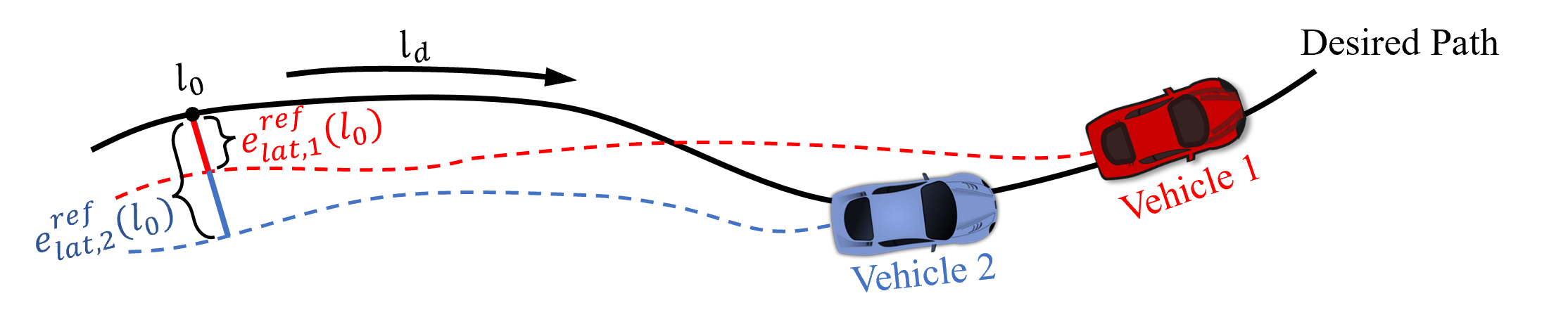}}
    \caption{Reference error computation under desired-path tracking}
    \label{fig: ref_error_desire}
\end{figure}

With V2V communication available, we can share additional valuable information beyond the desired path itself, enabling more advanced control strategy designs for better performance. In the V2V mode, we propose a novel \textbf{\textit{learn-from-predecessor (LFP)}} strategy to improve tracking error attenuation. This strategy is formulated directly in the spatial domain, parameterized by the desired path arc length, $l_d$. Its implementation is given by:
\begin{align}
\label{eq: ith_veh_control_learn}
&u_i(l_d)= 
-\mathbf{K_P} \mathbf{e}_i^{des}(l_d) - v_x \mathbf{K_D} \left( \mathbf{e}_i^{des}(l_d) \right)'
+ u_{l,i}(l_d), \notag \\[1.8ex]
&\text{where }~u_{l,i}(l_d)=\begin{cases}
     k_{ff} \left(\theta^{des}(l_d)\right)'&\text{for }i=1, \\ 
     \\
     u_{l,i-1}(l_d)+\mathbf{K_{LP}}\mathbf{y}_{i-1}(l_d) +\mathbf{K_{LD}} {\mathbf{y}}'_{i-1}(l_d)&\text{for }i=2,3,\ldots,m.
\end{cases}
\end{align}
Here, $u_{l,i}(l_d)$ denotes the learned control term of the $i^\text{th}$ vehicle. The lead vehicle's learned control term is identical to the feedforward term in the FF strategy. Each following vehicle ($i= 2,3,\ldots,m$) updates its predecessor's learned control term, $u_{l,i-1}(l_d)$, with two additional terms that depend on its predecessor’s output, $\mathbf{y}_{i-1}(l_d)$, and its derivative, ${\mathbf{y}}'_{i-1}(l_d)$, both recorded by the predecessor. The output is defined as $\mathbf{y}_{i-1}=\mathbf{C_{out}}\mathbf{e}_{i-1}^{des}$ where $\mathbf{C_{out}}$ is a selection matrix that extracts relevant components from the tracking error vector; this will be discussed in more detail later. The gains $\mathbf{K_{LP}}$ and $\mathbf{K_{LD}}$ represent the proportional and derivative learning gains, respectively. When the output is scalar, these learning gains are scalars; and when the output is a column vector, the gains are row vectors. The learned terms can be viewed as learning for feedforward control using predecessor error information. This structure enhances tracking performance while preserving the stability provided by the feedback terms.

Note that in the V2V mode, each following vehicle receives the desired path directly from its predecessor via V2V, along with the predecessor’s recorded output $\mathbf{y}_{i-1}(l_d)$, its derivative $\mathbf{y}'_{i-1}(l_d)$, and the learned control term $u_{l,i-1}(l_d)$. 
\begin{rem}
  While the tracking-error model is derived using the orthogonal projection of the C.G. onto a path, in controller implementation we use the closest point on the path instead to compute $e^{ref}_i$ and $\theta^{ref}_i$. For locally circular paths, this coincides with the orthogonal projection, and it is standard practice for general paths.
\end{rem}

\subsection{Inherent robustness against time delays}
\label{sec3.5}

      In contrast to the longitudinal case, lateral control of vehicle platoons is much less sensitive to communication or sensing delays. This is due to the inherent time gap introduced by the spatial separation between vehicles. As an illustrative case, consider the two-vehicle platoon shown in Fig.~\ref{fig: V2V delay}. At time $t = t_0$, let the arc-length positions of the predecessor and follower be $l_1$ and $l_2$, respectively, such that $l_1 - l_2 = \Delta l$. Suppose for illustration, the platoon is using the LFP strategy under V2V mode, and the predecessor transmits both the desired path information (i.e., $X^{des}(l_1), Y^{des}(l_1), $ and $\theta^{des}(l_1)$) and the learning information (i.e., $\mathbf{y}_{i-1}(l_1)$,  $\mathbf{y}'_{i-1}(l_1)$, and $u_{l,i-1}(l_1)$) at time $t_0$.

This information, however, is not needed by the follower immediately. Instead, it is only used when the follower reaches the spatial location \( l_1 \). Assuming the follower travels at longitudinal speed \( v_f \), it will arrive at \( l_1 \) approximately at time \( t_0 + \frac{\Delta l}{v_f} \). Therefore, as long as the communication delay is less than this spatially induced time gap $\frac{\Delta l}{v_f}$, the control strategy remains unaffected. The same idea holds for sensing delays when using the FF strategy under onboard sensing mode. This highlights a key advantage of the lateral platoon control problem formulated in this paper: it is naturally robust to moderate communication delays.
    \begin{figure}[h]
    \centering
    {\includegraphics[width=0.65\textwidth]{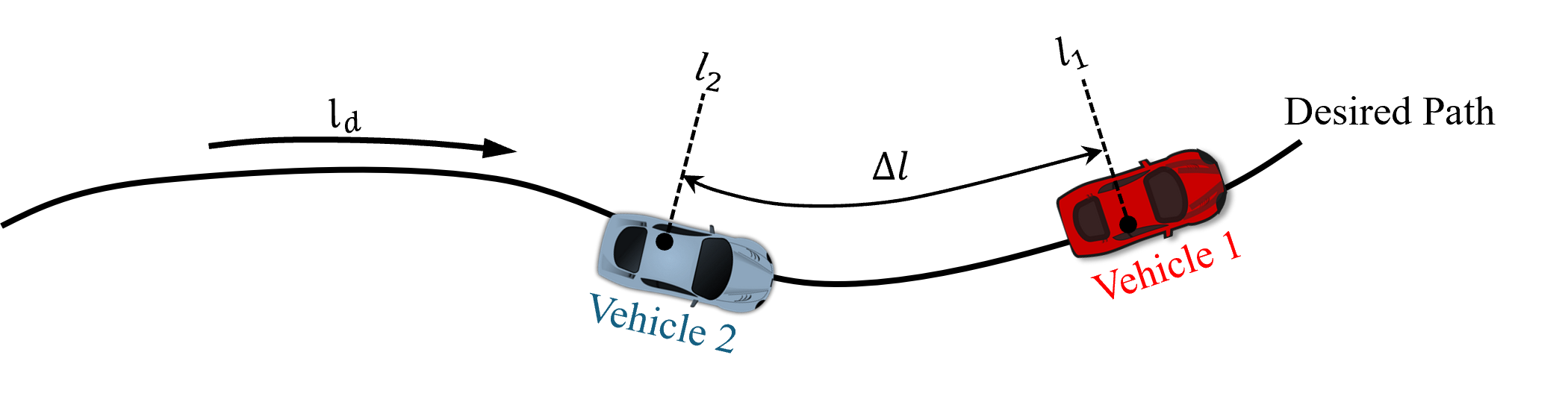}}
    \caption{Illustration of spatially induced time gap}
    \label{fig: V2V delay}
\end{figure}

\subsection{Problem formulation summary}
\label{sec3.7}
In the previous parts of this section, we introduced two control strategies of interest: the FF strategy using only onboard sensors and the LFP strategy utilizing V2V. For each strategy, the governing equations for analysis and V2V communication requirements were discussed. These are summarized in Table \ref{tab: control config}.




\begin{table}[t!]
\footnotesize
\centering
\caption{Summary of control strategies}
\label{tab: control config}
\renewcommand{\arraystretch}{1.6}

\begin{tabular}{>{\centering\arraybackslash}m{3.3cm} 
                >{\centering\arraybackslash}m{4cm} 
                >{\centering\arraybackslash}m{4cm} 
                >{\centering\arraybackslash}m{2.5cm}}
\hline \hline
Control Strategy & Governing Equation for Analysis  & V2V required \\
\hline

FF (Eq. \eqref{eq: i_th_u_reparam})  & Eqs. \eqref{eq: i_th_governing_reparam}, \eqref{eq: i_th_u_pred_track}  & \xmark \\

LFP (Eq. \eqref{eq: ith_veh_control_learn}) & Eqs. \eqref{eq: i_th_governing_reparam}, \eqref{eq: ith_veh_control_learn}  & \cmark \\
\hline \hline
\end{tabular}
\end{table}

\section{Lateral string stability}
\label{sec4}
In this section, we present the definition and conditions of lateral string stability. Specifically, Subsection \ref{sec 4.1} reviews string stability in the longitudinal context to motivate our definition of lateral string stability. In Subsection \ref{sec4.2}, we formally define $\mathcal{L}_2$ lateral string stability. In addition, we establish a necessary and sufficient condition for it. 
\subsection{Brief review of string stability}
\label{sec 4.1}
Intuitively speaking, a platoon is said to be string stable if errors or disturbances are not amplified along the
string. Various definitions of string stability have been proposed in the longitudinal context. To motivate the definition of lateral string stability proposed in this work, we first provide a brief review of existing string stability definitions.

Historically, the notion of string stability in the longitudinal control context can be traced back to \citep{chu1974decentralized}, and was later formalized and generalized in \citep{swaroop1994string} and \citep{swaroop2002string} using a Lyapunov-like $\epsilon$–$\delta$ definition. In this definition, the states of all vehicles in the platoon are uniformly bounded in vehicle index using the $\mathcal{L}_\infty$ norm, defined for a vector-valued function $z(t)$ as $\|z\|_{\mathcal{L}_\infty} = \sup_{t \geq 0} (\sup_i|z_i(t)|)$. Therefore, this definition is also referred to as $\mathcal{L}_\infty$ string stability (\citep{monteil2019string}) and was extended in \citep{besselink2017string} to incorporate external disturbances. Despite its practical relevance to safety (\citep{vegamoor2019review}), the $\mathcal{L}_\infty$ criterion poses challenges for controller synthesis. Even in the simple single-vehicle look-ahead case, ensuring $\mathcal{L}_\infty$ string stability requires the induced $\mathcal{L}_\infty$ gain of the error propagation map to be less than $1$ (\citep{swaroop1994string,besselink2017string}). For linear time-invariant (LTI), single-input-single-output (SISO) systems, this gain equals $\|h(t)\|_{\mathcal{L}_1}$, where $h(t)$ is the impulse response (\citep{desoer2009feedback}). Even in the LTI, SISO case, designing a controller to ensure  the weaker condition $\|h(t)\|_{\mathcal{L}_1} \le 1$ is nontrivial. \citep{swaroop1994string} (see also, \citep{darbha2018benefits}) showed that a sufficient condition is $h(t) \geq 0$ and $\sup_{\omega} |H(j\omega)| \leq 1$, where $H(j\omega)$ denotes the frequency response corresponding to $h(t)$. Ensuring $h(t)\geq0$ relates closely to the transient control problem in linear systems (\citep{lin1997nonovershooting,darbha2002controller}), which remains challenging in general.

Another important concept is $\mathcal{L}_2$ string stability, which requires the boundedness and attenuation of the $\mathcal{L}_2$ norm of errors (\citep{ploeg2013lp,vegamoor2019review}). For a vector-valued signal $z(t)$, this norm is given by $\|z\|_{\mathcal{L}_2} = \left(\int_0^\infty \scriptstyle{\sum\limits_i}|z_i(t)|^2dt\right)^{1/2}$. Although this requirement is weaker than $\mathcal{L}_\infty$ string stability (\citep{vegamoor2019review}), it greatly facilitates controller synthesis because for LTI systems, it is equivalent to a frequency-domain condition on the error propagation map, namely $\sup_{\omega} |H(j\omega)| \leq 1$. The $\mathcal{L}_2$ string stability has been widely adopted in the literature as a practical indicator of error attenuation (\citep{seiler2004disturbance,naus2010string,li2025nonlinear}), and has been further extended to accommodate diverse traffic scenarios (\citep{jin2016optimal,li2024sequencing,tian2025physically}) and various V2V communication conditions (\citep{darbha2018benefits,jin2014dynamics,ma2024selection}). Moreover, it has been shown in \citep{vegamoor2021string} that, under mild assumptions on the lead vehicle’s acceleration, $\mathcal{L}_2$ string stability implies guaranteed boundedness of the $\mathcal{L}_\infty$ norm of errors.

The preceding brief review highlights that the $\mathcal{L}_2$ string stability notion offers a favorable balance between theoretical rigor and practical applicability. It facilitates control synthesis while effectively capturing error attenuation along the platoon. Accordingly, we build upon the $\mathcal{L}_2$-based formulation for lateral string stability. 

\subsection{Definition of $\mathcal{L}_2$ lateral string stability}
\label{sec4.2}
To enable time-domain characterization, we consider the following definition of lateral string stability.


\begin{defn}[$\mathcal{L}_2$ Lateral String Stability]
\label{def:l2_lateral_string_stability}
Consider the platoon system in Eq.~\eqref{eq: i_th_governing_reparam} with zero initial states. For a given output $\mathbf{y}_i=\mathbf{C_{out}}\mathbf{e}_i^{des}$, the platoon is said to be \(\mathcal{L}_2\) lateral string stable with respect to the output \(\mathbf y_i\) if there exists a constant \(\gamma\in[0,1)\) such that
\begin{align}
\label{eq:strict_L2_string_stability}
\|\mathbf y_i\|_{\mathcal L_2}\le \gamma \|\mathbf y_{i-1}\|_{\mathcal L_2},\quad  \forall i=2,\dots,m.
\end{align}
\end{defn}

\begin{rem}
\label{rem: output choices}
We consider two choices of $\mathbf{y}_i$ in this paper. The first is
\[
\mathbf{y}_i = \mathbf{e}_i^{des} =
\begin{bmatrix}
e_{lat,i}^{des} \\[2pt]
\tilde{\theta}_i^{des}
\end{bmatrix},
\qquad
\mathbf{C}_{out} = \mathbf{I}_2,
\]
which captures both lateral and heading errors.

The second uses only the lateral error:
\[
\mathbf{y}_i = e_{lat,i}^{des},
\qquad
\mathbf{C}_{out} = [\,1\;\;0\,],
\]
a choice motivated by its direct safety relevance as it quantifies lateral deviation from the desired path.
\end{rem}






In the case of LTI systems, we can relate lateral string stability to frequency-domain conditions. In particular, the following theorem shows a necessary and sufficient condition for $\mathcal{L}_2$ lateral string stability when the vehicle interconnections can be expressed in the Laplace domain as input–output maps between $\mathbf{y}_{i-1}$ and $\mathbf{y}_i$. Note that for notational simplicity, we adopt lowercase symbols to denote Laplace-transformed variables; for instance, the Laplace transform of $z(l_d)$ is denoted by $z(s)$, where $s$ is the Laplace variable.

\begin{thm}
\label{prop: LFP DT frequency-domain condition}
Consider the platoon system in \eqref{eq: i_th_governing_reparam}, and suppose the interconnection between vehicles can be expressed as a causal, stable LTI system with transfer function matrix $\mathbf{H}_i(s)$, i.e.,
\begin{equation}
\mathbf{y}_i(s) = \mathbf{H}_i(s)\,\mathbf{y}_{i-1}(s), \quad \forall i=2,3,\ldots,m.
\end{equation}
The platoon is $\mathcal{L}_2$ lateral string stable with respect to the output $\mathbf{y}_i$ if and only if
\begin{equation}
\label{eq: freq_condition}
\|\mathbf{H}_i(s)\|_{\mathcal{H}_\infty} := \sup_{\omega\in\mathbb{R}} \sigma_1\!\left(\mathbf{H}_i(j\omega)\right) < 1, \quad \forall i=2,3,\ldots,m.
\end{equation}
\end{thm}
\begin{proof}
Since the mapping from $\mathbf y_{i-1}$ to $\mathbf y_i$ is a causal, stable LTI system, its induced $\mathcal L_2$ norm equals $\|\mathbf H_i\|_{\mathcal H_\infty}$ (\citep{bhattacharyya2018linear}). Thus,
$\|\mathbf y_i\|_{\mathcal L_2}\le \|\mathbf H_i\|_{\mathcal H_\infty}\|\mathbf y_{i-1}\|_{\mathcal L_2}.$
If $\|\mathbf H_i\|_{\mathcal H_\infty}<1$ for all $i=2,\dots,m$, then by choosing
$\gamma=\max_{i=2,\dots,m}\|\mathbf H_i\|_{\mathcal H_\infty}<1,$
Definition~\ref{def:l2_lateral_string_stability} holds.

Conversely, if Definition~\ref{def:l2_lateral_string_stability} holds, then for some $\gamma<1$,
$\|\mathbf y_i\|_{\mathcal L_2}\le \gamma \|\mathbf y_{i-1}\|_{\mathcal L_2}$ for all $i=2,\dots,m.$
Taking the supremum over all nonzero $\mathbf y_{i-1}\in\mathcal L_2$ yields
$\|\mathbf H_i\|_{\mathcal H_\infty}
=\sup_{\mathbf y_{i-1}\neq 0}\frac{\|\mathbf y_i\|_{\mathcal L_2}}{\|\mathbf y_{i-1}\|_{\mathcal L_2}}
\le \gamma<1,$
for all $i=2,\dots,m.$
\end{proof}
\begin{rem}
When the output $\mathbf{y}_i$ is scalar-valued, the transfer function $\mathbf{H}_i(s)$ reduces to a scalar transfer function $H_i(s)$ rather than a matrix. In this case, the $\mathcal{H}_\infty$ norm definition simplifies to the standard SISO form
\begin{equation}
\|H_i(s)\|_{\mathcal{H}_\infty} := \sup_{\omega \in \mathbb{R}} \left| H_i(j\omega) \right|,
\end{equation}
where the largest singular value $\sigma_1(\cdot)$ is replaced by the absolute value.
\end{rem}

\section{$\mathcal{L}_2$ lateral string stability analysis}
\label{sec5}
In this section, we analyze the $\mathcal{L}_2$ lateral string stability of the two control strategies of interest: the FF strategy under onboard sensing mode and the LFP strategy under V2V mode. The two strategies are analyzed in Subsections \ref{sec5.1} and \ref{sec5.2}, respectively. As discussed in Remark~\ref{rem: output choices}, we consider two output definitions: $\mathbf{y}_i = \mathbf{e}_i^{des}$ and $\mathbf{y}_i = e_{lat,i}^{des}$. The main results are summarized in Table~\ref{tab: string stabilty summary}. In the table, a cross mark (\xmark) indicates that $\mathcal{L}_2$ lateral string stability cannot be achieved for the corresponding control strategy and output definition, whereas a check mark (\cmark) indicates that it can be achieved. Notably, among the combinations considered, $\mathcal{L}_2$ lateral string stability can be achieved only for the LFP strategy when the output is defined as $e_{lat,i}^{des}$.
\begin{table}[h]
\centering
\footnotesize
\caption{$\mathcal{L}_2$ lateral string stability results under different control strategies and output definitions}
\begin{tabular}{ccc}
\hline \hline
\textbf{Control strategy} & $\mathbf{y}_i = \mathbf{e}_i^{des}$ & $\mathbf{y}_i = e_{lat,i}^{des}$ \\
\hline
FF  & \xmark & \xmark \\
LFP  & \xmark & \cmark \\
\hline \hline
\end{tabular}
\label{tab: string stabilty summary}
\end{table}

As shown in Theorem \ref{prop: LFP DT frequency-domain condition}, the Laplace domain representation facilitates analysis for LTI systems. Hence, we apply the Laplace transform to the governing equations. As noted in Table~\ref{tab: control config}, Eq.~\eqref{eq: i_th_governing_reparam} captures the plant dynamics in response to the control input $u_i$, and its structure remains the same across both control strategies. The expression of $u_i$ differs between the strategies. Assuming zero initial conditions, the Laplace-domain representation of Eq.~\eqref{eq: i_th_governing_reparam} is given by:
\begin{equation}
\label{model Laplace}
    \underbrace{(s^2v_x^2\mathbf{M} + sv_x\mathbf{C}  + \mathbf{L})}_{\mathbf{\hat{M}}(s)}\mathbf{e}^{des}_i(s) = \mathbf{B} {u_i}(s) - s\mathbf{F}\theta^{des}(s),
\end{equation}
the Laplace-domain expressions for $u_i(s)$ under each strategy will be provided in the corresponding subsections.

Before delving into the main analysis, we establish some results that will simplify and unify several arguments to come.
We begin by a useful lemma that provides a perturbation bound on singular values, obtained from Weyl’s inequalities and the Hoffman–Wielandt theorem (\cite{horn2012matrix}). Note that $\mathbb{M}_{n,m}$ denotes the set of all complex $n\times m$ matrices.
\begin{lem}[Corollary 7.3.5 in \citep{horn2012matrix}]
\label{lem: singular value bound}
Let $\mathbf{P}, \mathbf{Q} \in \mathbb{M}_{n,m}$ and let $q = \min\{m, n\}$. Let 
$\sigma_1(\mathbf{P}) \geq \cdots \geq \sigma_q(\mathbf{P})$ and 
$\sigma_1(\mathbf{Q}) \geq \cdots \geq \sigma_q(\mathbf{Q})$ be the nonincreasingly ordered singular values of $\mathbf{P}$ and $\mathbf{Q}$, respectively. Then
\begin{equation}
|\sigma_j(\mathbf{P}) - \sigma_j(\mathbf{Q})| \leq \sigma_1\left(\mathbf{P} - \mathbf{Q}\right) \quad \text{for each }j = 1, \ldots, q.
\end{equation}
\end{lem}

Using Lemma~\ref{lem: singular value bound}, we establish the following theorem, which will be used repeatedly to show that $\mathcal{L}_2$ lateral string stability cannot be achieved with respect to the output $\mathbf{y}_i = \mathbf{e}_i^{\text{des}}$.
\begin{thm}
\label{prop: multi-error not stable general}
    Suppose there exists an $i \in \{2, \dots, m\}$ such that we have the input-output relationship
    \begin{equation}
\mathbf{e}_{i}^{des}(s) = \mathbf{H}_i(s)\, \mathbf{e}_{i-1}^{des}(s),
\end{equation}
where $\mathbf{H}_i(s) = \mathbf{I}_2 + \mathbf{R}(s)$, with $\mathbf{I}_2$ denoting the $2 \times 2$ identity matrix and $\mathbf{R}(s)$ a rank-deficient matrix for all $s$. Then the platoon cannot be $\mathcal{L}_2$ lateral string stable with respect to the output $\mathbf{y}_i = \mathbf{e}_i^{des}$.
\end{thm}
\begin{proof}
    Substituting $s=j\omega$, we obtain the frequency-domain representation of the input-output relationship
\begin{align}
\mathbf{e}^{des}_{i}(j\omega) = \mathbf{H}_i(j\omega)\mathbf{e}^{des}_{i-1}(j\omega).
\end{align}
Since $\mathbf{R}(s)$ is rank-deficient for all $s$, 
\begin{equation}
    \sigma_2\left(\mathbf{R}(j\omega)\right)=0,~~~~\forall\omega\in\mathbb{R}.
\end{equation}
From Lemma \ref{lem: singular value bound}, we have
\begin{align}
    \sigma_1\left(\mathbf{H}_i(j\omega)\right)=\sigma_1\left(\mathbf{I}_2+\mathbf{R}(j\omega)\right)\geq\left|\sigma_2\left(\mathbf{I}_2\right)-\sigma_2\left(-\mathbf{R}(j\omega)\right)\right|=1, ~~~~\forall\omega\in\mathbb{R}.
\end{align}
Therefore, the $\mathcal{H}_\infty$ norm of $\mathbf{H}_i(s)$ satisfies 
\begin{equation}
    \|\mathbf{H}_i(s)\|_{\mathcal{H}_\infty}:=\sup_{\omega\in\mathbb{R}}\sigma_1\left(\mathbf{H}_i(j\omega)\right)\geq1,
\end{equation}
which violates the condition in Theorem \ref{prop: LFP DT frequency-domain condition}; hence, the platoon cannot be $\mathcal{L}_2$ lateral string stable with respect to the output $\mathbf{y}_i = \mathbf{e}_i^{des}$.
\end{proof}

\subsection{Feedback-feedforward strategy under onboard sensing mode}
\label{sec5.1}
In this subsection, we conduct lateral string stability analysis for the FF strategy.

Recall from Table \ref{tab: control config} that for the FF strategy, the control law formulation that is convenient for analysis is given in Eq. \eqref{eq: i_th_u_pred_track}. Applying the Laplace transform to both sides of it, we obtain 
\begin{align}
\label{eq: u Laplace FF PT}
u_i(s) =
\left\{
\begin{array}{ll}
-(\mathbf{K_P} + s v_x \mathbf{K_D})\, \mathbf{e}_i^{des}(s) + s k_{ff} \theta^{des}(s),  &\text{for } i = 1, \\[3.5ex] 
-\smash{(\underbrace{\mathbf{K_P} + s v_x \mathbf{K_D}}_{\mathbf{K_{fb}}(s)}}) 
\left( \mathbf{e}_i^{des}(s) - \mathbf{e}_{i-1}^{des}(s) \right)
+ s k_{ff} \left( \theta^{des}(s) + [0~~1] \mathbf{e}_{i-1}^{des}(s) \right),  &\text{for } i = 2,3,\ldots,m.
\end{array}
\right.  \\ \notag
\end{align}
For the FF strategy, it suffices to analyze the error propagation between the first and second vehicle in the platoon to show instability. Substituting the case of $i=1$ above into Eq. \eqref{model Laplace} yields
\begin{equation}
\label{eq: Laplace FF PT i=1}
    \left(\mathbf{\hat{M}}(s)+\mathbf{BK_{fb}}(s)\right)\mathbf{e}^{des}_1(s) =  \left(s k_{ff}\mathbf{B} - s\mathbf{F}\right)\theta^{des}(s),
\end{equation}
Similarly, substituting the case of $i=2$ into Eq. \eqref{model Laplace} yields
\begin{equation}
\label{eq: Laplace FF PT i=2}
    \left(\mathbf{\hat{M}}(s)+\mathbf{BK_{fb}}(s)\right)\mathbf{e}^{des}_2(s) =  \mathbf{B}\left(\mathbf{K_{fb}}(s)+s k_{ff}[0~~1]\right)\mathbf{e}^{des}_{1}(s) +\left(s k_{ff}\mathbf{B} - s\mathbf{F}\right)\theta^{des}(s),
\end{equation}
 Subtracting the two sides of Eq. \eqref{eq: Laplace FF PT i=1} from the two sides of Eq. \eqref{eq: Laplace FF PT i=2}, we obtain
   \begin{equation}
    \left(\mathbf{\hat{M}}(s)+\mathbf{BK_{fb}}(s)\right)\left(\mathbf{e}^{des}_2(s)-\mathbf{e}^{des}_1(s)\right) =  \mathbf{B}\left(\mathbf{K_{fb}}(s)+s k_{ff}[0~~1]\right)\mathbf{e}^{des}_{1}(s),
\end{equation}
Rearranging the above equation:
 \begin{align}
 \label{eq: FF PT transfer function}
\mathbf{e}^{des}_{2}(s) = \underbrace{\left(\mathbf{I}_2+\left({\mathbf{\hat{M}}(s)}+\mathbf{B}\mathbf{K_{fb}}(s)\right)^{-1}\mathbf{B}\left(\mathbf{K_{fb}}(s)+s k_{ff}[0~~1]\right)\right)}_{\mathbf{H}_2(s)}\mathbf{e}^{des}_{1}(s).
\end{align}

\subsubsection{Lateral string stability analysis with respect to the output $\mathbf{y}_i = \mathbf{e}_i^{des}$}
For the output definition $\mathbf{y}_i=\mathbf{e}_i^{des}$, the following proposition shows that $\mathcal{L}_2$ lateral string stability cannot be achieved by the FF strategy.
\begin{prop}
\label{prop: FF PT y=E notstable}
    The FF strategy under onboard sensing mode cannot be $\mathcal{L}_2$ lateral string stable with respect to the output $\mathbf{y}_i = \mathbf{e}_i^{des}$.
\end{prop}
\begin{proof}
In Eq. \eqref{eq: FF PT transfer function}, note that $\mathbf{B}\left(\mathbf{K_{fb}}(s)+s k_{ff}[0~~1]\right)$ is the product of a column vector and a row vector. Hence, 
\begin{equation}
\text{rank}\left(\mathbf{B}\left(\mathbf{K_{fb}}(s)+s k_{ff}[0~~1]\right)\right)\leq1,~~~~\forall s\in\mathbb{C}.
\end{equation}
 Using the identity $\text{rank}(\mathbf{PQ})\leq\min\left(\text{rank}(\mathbf{P}),\text{rank}(\mathbf{Q})\right)$, we have
 \begin{equation}
\text{rank}\Big(\underbrace{\left({\mathbf{\hat{M}}(s)}+\mathbf{B}\mathbf{K_{fb}}(s)\right)^{-1}\mathbf{B}\left(\mathbf{K_{fb}}(s)+sk_{ff}[0~~1]\right)}_{\mathbf{R}(s)}\Big)\leq1,~~~~\forall s\in\mathbb{C}.
\end{equation}
Therefore, in Eq. \eqref{eq: FF PT transfer function}, $\mathbf{H}_2(s)=\mathbf{I}_2+\mathbf{R}(s)$, and $\mathbf{R}(s)$ is rank-deficient. Using Theorem \ref{prop: multi-error not stable general}, the proof is complete.
\end{proof}

\subsubsection{Lateral string stability analysis with respect to the output $\mathbf{y}_i = e_{lat,i}^{des}$}
For the output definition $\mathbf{y}_i = e_{lat,i}^{des}$, by pre-multiplying both sides of Eq. \eqref{eq: FF PT transfer function} with $[1~~0]$, we can extract the lateral errors $e^{des}_{lat,1}(s)$ from the error vector $\mathbf{e}^{des}_{1}(s)$ on the left-hand side: 
\begin{align}
\label{eq: FF PT SISO transfer}
[1~~0]\mathbf{e}^{des}_{2}(s) =& [1~~0]\left(\mathbf{I}_2+\left({\mathbf{\hat{M}}(s)}+\mathbf{B}\mathbf{K_{fb}}(s)\right)^{-1}\mathbf{B}\left(\mathbf{K_{fb}}(s)+s k_{ff}[0~~1]\right)\right)\mathbf{e}^{des}_{1}(s) \notag\\[1.5ex]
\implies{e}^{des}_{lat,2}(s) =& e^{des}_{lat,1}(s){+[1~~0]\left({\mathbf{\hat{M}}(s)}+\mathbf{B}\mathbf{K_{fb}}(s)\right)^{-1}\mathbf{B}\left(\mathbf{K_{fb}}(s)+s k_{ff}[0~~1]\right)}\mathbf{e}^{des}_{1}(s) \notag \\
=&\left(1{+[1~~0]\left({\mathbf{\hat{M}}(s)}+\mathbf{B}\mathbf{K_{fb}}(s)\right)^{-1}\mathbf{B}\left(\mathbf{K_{fb}}(s)+s k_{ff}[0~~1]\right)}\begin{bmatrix}
    1 \\ 0
\end{bmatrix}\right)e^{des}_{lat,1}(s)  \notag \\
&+\underbrace{\left({[1~~0]\left({\mathbf{\hat{M}}(s)}+\mathbf{B}\mathbf{K_{fb}}(s)\right)^{-1}\mathbf{B}\left(\mathbf{K_{fb}}(s)+s k_{ff}[0~~1]\right)}\begin{bmatrix}
    0 \\ 1
\end{bmatrix}\right)}_{H_2(s)}\tilde{\theta}_1^{des}(s).
\end{align}

The following proposition shows that $\mathcal{L}_2$ lateral string stability with respect to $\mathbf{y}_i = e_{lat,i}^{des}$ cannot be achieved by the FF strategy either, using a counterexample.
\begin{prop}
\label{prop: FF PT y=e notstable}
    The FF strategy under onboard sensing mode cannot be $\mathcal{L}_2$ lateral string stable with respect to the output $\mathbf{y}_i = {e}_{lat,i}^{des}$.
\end{prop}
\begin{proof}
See Appendix \ref{apped A}.
\end{proof}

\subsection{Learn-from-predecessor strategy under V2V mode}
\label{sec5.2}
In this subsection, we conduct lateral string stability analysis for the LFP strategy.

Recall from Table \ref{tab: control config} that for the LFP strategy, the control law is given in Eq. \eqref{eq: ith_veh_control_learn}. Applying the Laplace transform to both sides of it, we obtain  
\begin{align}
&u_i(s)= 
-(\underbrace{\mathbf{K_P} + sv_x \mathbf{K_D}}_{\mathbf{K_{fb}}(s)}) \mathbf{e}_i^{des}(s)
+ u_{l,i}(s), \notag \\[1ex]
&\text{where }~u_{l,i}(s)=\begin{cases}
     sk_{ff} \theta^{des}(s)&\text{for }i=1, \\ 
     \\
     u_{l,i-1}(s)+(\smash{\underbrace{\mathbf{K_{LP}}+s\mathbf{K_{LD}}}_{\mathbf{K_L}(s)}})\mathbf{y}_{i-1}(s) &\text{for }i=2,\ldots,m. 
\end{cases}
\\ \notag
\end{align}
Substituting the above into Eq. \eqref{model Laplace}, for $i=2,3,\ldots,m$, we have
\begin{align}
\label{eq: desired learn error i-1}
\left({\mathbf{\hat{M}}(s)}+\mathbf{B}\mathbf{K_{fb}}(s)\right)\mathbf{e}^{des}_{i-1}(s) = \mathbf{B}u_{l,i-1}(s) - s\mathbf{F}\theta^{des}(s),
\end{align}
\begin{align}
\label{eq: desired learn error i}
\left({\mathbf{\hat{M}}(s)}+\mathbf{B}\mathbf{K_{fb}}(s)\right)\mathbf{e}^{des}_{i}(s) = \mathbf{B}\left(u_{l,i-1}(s)+\mathbf{K_L}(s)\mathbf{y}_{i-1}(s)\right) - s\mathbf{F}\theta^{des}(s).
\end{align}
Subtracting the two sides of Eq. \eqref{eq: desired learn error i-1} from the two sides of Eq. \eqref{eq: desired learn error i}, it is obtained that
\begin{align}
\left({\mathbf{\hat{M}}(s)}+\mathbf{B}\mathbf{K_{fb}}(s)\right)\left(\mathbf{e}^{des}_{i}(s)-\mathbf{e}^{des}_{i-1}(s)\right) = \mathbf{B}\mathbf{K_L}(s)\mathbf{y}_{i-1}(s),~~\forall i=2,3,\ldots,m.
\end{align}
Further rearranging the above equation:
\begin{align}
\label{eq: transfer function LFP DT}
\mathbf{e}^{des}_{i}(s) = \mathbf{e}^{des}_{i-1}(s)+\left({\mathbf{\hat{M}}(s)}+\mathbf{B}\mathbf{K_{fb}}(s)\right)^{-1}\mathbf{B}\mathbf{K_L}(s)\mathbf{y}_{i-1}(s),~~\forall i=2,3,\ldots,m.
\end{align}
\subsubsection{Lateral string stability analysis with respect to the output $\mathbf{y}_i = \mathbf{e}_i^{des}$}
For $\mathbf{y}_i = \mathbf{e}_i^{des}$, Eq. \eqref{eq: transfer function LFP DT} becomes
\begin{equation}
\label{eq: LEP DT y=E transfer function}
    \mathbf{e}^{des}_{i}(s) = \Big(\underbrace{\mathbf{I}_2+\left({\mathbf{\hat{M}}(s)}+\mathbf{B}\mathbf{K_{fb}}(s)\right)^{-1}\mathbf{B}\mathbf{K_L}(s)}_{\mathbf{H}(s)}\Big)\mathbf{e}^{des}_{i-1}(s),~~\forall i=2,3,\ldots,m.
\end{equation}
Note that for $\mathbf{y}_i = \mathbf{e}_i^{des}$, the corresponding $\mathbf{K_L}$ is a $1\times2$ row vector. The following proposition shows that $\mathcal{L}_2$ lateral string stability cannot be achieved by the LFP strategy.
\begin{prop}
\label{prop: LFP DT y=E not stable}
    The LFP strategy under V2V mode cannot be $\mathcal{L}_2$ lateral string stable with respect to the output $\mathbf{y}_i = \mathbf{e}_i^{des}$.
\end{prop}
\begin{proof}
 In Eq. \eqref{eq: LEP DT y=E transfer function}, $\mathbf{B}\mathbf{K_{L}}(s)$ is the product of a column vector and a row vector. Hence, 
\begin{equation}
\text{rank}\left(\mathbf{B}\mathbf{K_{L}}(s)\right)\leq1,~~~~\forall s\in\mathbb{C}.
\end{equation}
 Using the identity $\text{rank}(\mathbf{PQ})\leq\min\left(\text{rank}(\mathbf{P}),\text{rank}(\mathbf{Q})\right)$, we have
 \begin{equation}
\text{rank}\Big(\underbrace{\left({\mathbf{\hat{M}}(s)}+\mathbf{B}\mathbf{K_{fb}}(s)\right)^{-1}\mathbf{B}\mathbf{K_{L}}(s)}_{\mathbf{R}(s)}\Big)\leq1,~~~~\forall s\in\mathbb{C}.
\end{equation}
Therefore, in Eq. \eqref{eq: LEP DT y=E transfer function}, $\mathbf{H}(s)=\mathbf{I}_2+\mathbf{R}(s)$, and $\mathbf{R}(s)$ is rank-deficient. Using Theorem \ref{prop: multi-error not stable general}, the proof is complete.
\end{proof}

\subsubsection{Lateral string stability analysis with respect to the output $\mathbf{y}_i = e_{lat,i}^{des}$}
For $\mathbf{y}_i=e_{lat,i}^{des}$, Eq. \eqref{eq: transfer function LFP DT} becomes
\begin{equation}
    \mathbf{e}^{des}_{i}(s) = \mathbf{e}^{des}_{i-1}(s)+\left({\mathbf{\hat{M}}(s)}+\mathbf{B}\mathbf{K_{fb}}(s)\right)^{-1}\mathbf{B}\mathbf{K_L}(s)e_{lat,i-1}^{des}(s),~~\forall i=2,3,\ldots,m.
\end{equation}
Pre-multiplying $[1~~0]$ on both sides of the above equation extracts the lateral error from the error vector:
\begin{align}
\label{eq: LFP DT siso}
    [1~~0]\mathbf{e}^{des}_{i}(s) &= [1~~0]\mathbf{e}^{des}_{i-1}(s)+[1~~0]\left({\mathbf{\hat{M}}(s)}+\mathbf{B}\mathbf{K_{fb}}(s)\right)^{-1}\mathbf{B}\mathbf{K_L}(s)e_{lat,i-1}^{des}(s) \notag \\
    \implies e_{lat,i}^{des}(s)&=\underbrace{\left(1+[1~~0]\left({\mathbf{\hat{M}}(s)}+\mathbf{B}\mathbf{K_{fb}}(s)\right)^{-1}\mathbf{B}\mathbf{K_L}(s)\right)}_{H(s)}e_{lat,i-1}^{des}(s)
\end{align}
In this case, $\mathcal{L}_2$ lateral string stability can, in fact, be achieved. However, additional insight into the controller structure can be gained via further analysis. For this purpose, we recall a result that will be useful in the next proposition, namely Bode’s sensitivity integral constraint, stated as follows.
\begin{lem}[Bode's Sensitivity Integral, \cite{goodwin2001control,emami2019bode}]
    \label{lem: Bode's}
    For any SISO closed-loop stable LTI system, let $G(s)$ denote the open-loop transfer function and ${S}(s):=\frac{1}{1+G(s)}$ denote the sensitivity function. If $G(s)$ is stable and strictly proper, then
    \begin{equation}
        \int_{0}^{\infty} \ln |S(j\omega)| \, d\omega =  - \frac{\pi}{2} \lim_{s \to \infty} sG(s).
    \end{equation}
\end{lem}

In the next proposition, we demonstrate why it is crucial in the LFP strategy to learn from the predecessor’s lateral error derivative $\left(e_{lat,i-1}^{des}\right)'$, by employing the derivative learning gain $\mathbf{K_{LD}}$. Note that the proposition is based on a mild assumption that $H(s)$ does not have any right-half plane (RHP) zero, which is typically satisfied and can be verified numerically for practical parameters. Here, we recall that $\mathbf{\hat{M}}(s)=s^2v_x^2\mathbf{M} + sv_x\mathbf{C}  + \mathbf{L}$, $\mathbf{K_{fb}}(s)=\mathbf{K_P} + s v_x \mathbf{K_D}$, and $\mathbf{K_L}=\mathbf{K_{LP}}+s\mathbf{K_{LD}}$. Also, note that for $\mathbf{y}_i=e_{lat,i}^{des}$, the corresponding learning gains $\mathbf{K_{LP}}$ and $\mathbf{K_{LD}}$ are scalars. 
\begin{prop}
\label{prop: LFP DT K_ld=0 not stable}
    Suppose $H(s)$ defined in Eq. \eqref{eq: LFP DT siso} does not have any RHP zero. If the derivative learning gain, $\mathbf{K_{LD}}$, equals zero, then the LFP strategy under V2V mode cannot be $\mathcal{L}_2$ lateral string stable with respect to the output $\mathbf{y}_i=e_{lat,i}^{des}$.
\end{prop}
\begin{proof}
    Let 
    \begin{equation}
        G(s)=[1~~0]\left({\mathbf{\hat{M}}(s)}+\mathbf{B}\mathbf{K_{fb}}(s)\right)^{-1}\mathbf{BK_L}(s),
    \end{equation}
 so that $H(s)=1+G(s)$ in Eq. \eqref{eq: LFP DT siso}. In the above expression, by expansion,
\begin{align}
\label{eq: G(s) denominator degree}
&{\mathbf{\hat{M}}(s)}+\mathbf{B}\mathbf{K_{fb}}(s) \notag \\
=
&\begin{bmatrix}
 s^2 m v_x^2 + (C_f k_{\dot e_{lat}} v_x+C_f+C_r)+C_f k_{e_{lat}} & s (a C_f-b C_r+C_f k_{\dot{\tilde{\theta}}} v_x)+C_f
   (k_{\tilde{\theta}}-1)-C_r \\[1.5ex]
 s \left(a (C_f k_{\dot e_{lat}} v_x+C_f)-b C_r\right)+a C_f k_{e_{lat}} & s^2 I_z v_x^2 + s (a^2 C_f+a C_f k_{\dot{\tilde{\theta}}} v_x+b^2 C_r)+a C_f
   (k_{\tilde{\theta}}-1)+b C_r \\
\end{bmatrix}.
\end{align}
Note that if $\mathbf{K_{LD}}=0$, then $\mathbf{K_L}(s)=\mathbf{K_{LP}}$, we have
\begin{align}
\label{eq: LFP DT K_ld=0 G(s)}
G(s)=\frac{[1~~0]\text{adj}\left({\mathbf{\hat{M}}(s)}+\mathbf{B}\mathbf{K_{fb}}(s)\right)\mathbf{BK_{LP}}}{\text{det}\left({\mathbf{\hat{M}}(s)}+\mathbf{B}\mathbf{K_{fb}}(s)\right)} 
=\frac{
    \mathbf{K_{LP}}C_f\left(s^2 I_z v_x^2 + s \left(b^2 C_r+ab C_r\right)+b C_r  +aC_r\right)}
{\text{det}\left({\mathbf{\hat{M}}(s)}+\mathbf{B}\mathbf{K_{fb}}(s)\right)},
\end{align}
From Eq. \eqref{eq: G(s) denominator degree}, it is clear that $\text{det}\left({\mathbf{\hat{M}}(s)}+\mathbf{B}\mathbf{K_{fb}}(s)\right)$ is a quartic polynomial of $s$, which means that $G(s)$ has a relative degree of 2.

Since $H(s)$ does not have any RHP zero, $S(s)=\frac{1}{H(s)}=\frac{1}{1+G(s)}$ is stable. This allows us to construct a "virtual" closed-loop stable system, with $G(s)$ being the open-loop transfer function, as shown in Fig. \ref{fig: closed_loop}. Note that this "virtual" system does not have much physical meaning, but it allows us to invoke Lemma \ref{lem: Bode's}. 
\begin{figure}[htb!]
    \centering
    \includegraphics[width = 0.2\linewidth]{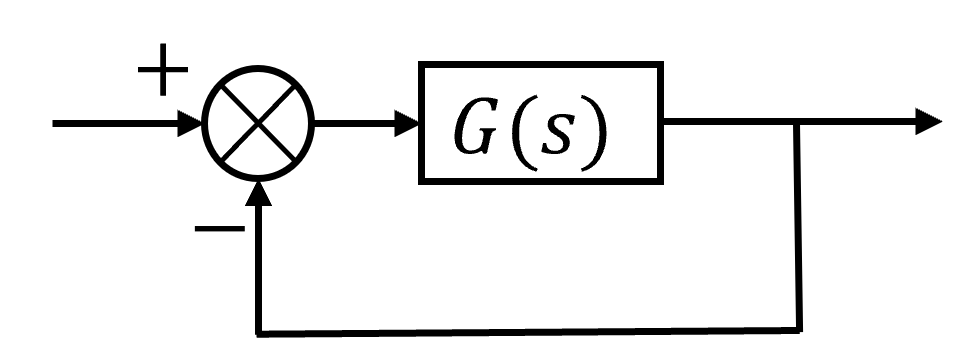}
    \caption{The "virtual" closed-loop system}
    \label{fig: closed_loop}
\end{figure}

Since $\mathbf{K_{fb}}$ is designed to stabilize the error vector's evolution, we have all the roots of $\text{det}\left({\mathbf{\hat{M}}(s)}+\mathbf{B}\mathbf{K_{fb}}(s)\right)$ in the left-half plane and $G(s)$ being stable for any practical controller. It has been shown from Eq. \eqref{eq: LFP DT K_ld=0 G(s)} that $G(s)$ is proper with a relative degree of 2. Hence, $sG(s)$ has a relative degree of 1. By Lemma \ref{lem: Bode's},
\begin{equation}
\label{eq: LFP DT sensitivity}
        \int_{0}^{\infty} \ln |S(j\omega)| \, d\omega =  - \frac{\pi}{2} \lim_{s \to \infty} sG(s)=0.
    \end{equation}
Recall that $S(s)=\frac{1}{H(s)}$, we have
\begin{equation}
\label{eq: LFP FT H(s)}
        \int_{0}^{\infty} \ln |H(j\omega)| \, d\omega = \int_{0}^{\infty} \ln \left|\frac{1}{S(j\omega)}\right| \, d\omega=\int_{0}^{\infty} \ln \frac{1}{\left|S(j\omega)\right|} \, d\omega=\int_{0}^{\infty} -\ln \left|{S(j\omega)}\right| \, d\omega=0.
    \end{equation}
The above equation indicates that if there exist some frequencies where $|H(j\omega)|<1$ (which indicates $\ln|H(j\omega)|<0$), there must be some other frequencies where $|H(j\omega)|>1$ (which indicates $\ln|H(j\omega)|>0$). 

Therefore, the $\mathcal{H}_\infty$ norm of ${H}(s)$ satisfies 
\begin{equation}
    \|{H}(s)\|_{\mathcal{H}_\infty}:=\sup_{\omega\in\mathbb{R}}\left|{H}(j\omega)\right|\geq1.
\end{equation}
which violates the condition in Theorem \ref{prop: LFP DT frequency-domain condition}, and the proof is complete.
\end{proof}

\subsubsection{Controller design of learn-from-predecessor}
Proposition \ref{prop: LFP DT K_ld=0 not stable} highlights the critical role of the derivative learning gain $\mathbf{K_{LD}}$. When $\mathbf{K_{LD}}\neq0$, it becomes possible to suppress $\left\|H(s)\right\|_{\mathcal{H}_\infty}$ to be strictly less than one, and thereby ensure lateral string stability. Moreover, the steps in the proof of Proposition \ref{prop: LFP DT K_ld=0 not stable} can be instructive for the design of $\mathbf{K_{LD}}$. Specifically, for $\mathbf{K_{LD}}\neq0$, Eq. \eqref{eq: LFP DT K_ld=0 G(s)} becomes
\begin{align}
G(s)
=\frac{
    \left(\mathbf{K_{LP}}+\mathbf{K_{LD}}s\right)C_f\left(s^2 I_z v_x^2 + s \left(b^2 C_r+ab C_r\right)+b C_r  +aC_r\right)}
{\text{det}\left({\mathbf{\hat{M}}(s)}+\mathbf{B}\mathbf{K_{fb}}(s)\right)},
\end{align}
which has a relative degree of 1. Expanding the above equation, we can obtain that Eq. \eqref{eq: LFP DT sensitivity} becomes
\begin{equation}
        \int_{0}^{\infty} \ln |S(j\omega)| \, d\omega =  - \frac{\pi}{2} \lim_{s \to \infty} sG(s)=- \frac{\pi}{2}\frac{C_f\mathbf{K_{LD}}}{mv_x^2}.
    \end{equation}
If we design $\mathbf{K_{LD}}<0$, then by Eq. \eqref{eq: LFP FT H(s)}, 
\begin{equation}
        \int_{0}^{\infty} \ln |H(j\omega)| \, d\omega =\int_{0}^{\infty} -\ln \left|{S(j\omega)}\right| \, d\omega<0.
    \end{equation}
Hence, it is possible to achieve $|H(j\omega)|<1$ for every frequency.

Another value that is instructive for controller design is $H(0)$, since $|H(0)|<1$ is essential for $\|H(s)\|_{\mathcal{H}_\infty}<1$. Expanding the expression of $H(s)$, defined in Eq. \eqref{eq: LFP DT siso}, and evaluating it at $s=0$, we obtain
\begin{equation}
    H(0)=\frac{k_{e_{lat}}+\mathbf{K_{LP}}}{k_{e_{lat}}}.
\end{equation}
To make $|H(0)|<1$, we must choose $-2k_{e_{lat}}<\mathbf{K_{LP}}<0$.

Note that although the conditions $\mathbf{K_{LD}}<0$ and $-2k_{e_{lat}}<\mathbf{K_{LP}}<0$ are useful for controller design, they are necessary but not sufficient conditions. To ensure $\mathcal{L}_2$ lateral string stability, the frequency-domain condition $\|H(s)\|_{\mathcal{H}_\infty}<1$ given by Theorem \ref{prop: LFP DT frequency-domain condition} must be satisfied. A more numerically tractable sufficient condition that guarantees the frequency-domain requirement can be derived as follows. Let $N(s)$ and $D(s)$ denote the numerator and denominator of $H(s)$, respectively, then
\begin{align}
    \label{eq: string stable sufficient condition}\|H(j\omega)\|_{\mathcal{H}_\infty}<1\iff
    \frac{\left|N(j\omega)\right|}{\left|D(j\omega)\right|}<1,\forall \omega\in\mathbb{R} \iff
    \left|D(j\omega)\right|^2-\left|N(j\omega)\right|^2>0,\forall \omega\in\mathbb{R}.
\end{align}
Substituting the explicit expression of $H(s)$ defined in Eq.~\eqref{eq: LFP DT siso} and letting $s = j\omega$, we can obtain
\begin{equation}
    |D(j\omega)|^2 - |N(j\omega)|^2 = a_6 \omega^6 + a_4 \omega^4 + a_2 \omega^2 + a_0,
\end{equation}
where the coefficients $a_6$, $a_4$, $a_2$, and $a_0$ depend on the vehicle parameters and controller gains. Noting that $\omega^6$, $\omega^4$, and $\omega^2$ are nonnegative for all $\omega \in \mathbb{R}$, a sufficient condition for ensuring $\|H(j\omega)\|_{\mathcal{H}_\infty} < 1$ is
\begin{equation}
\label{eq: string stable coefficient condition}
    a_6 > 0, \quad a_4 > 0, \quad a_2 > 0, \quad a_0 > 0.
\end{equation}

\begin{rem}
\label{rem: design}
    Based on the above discussion, the design procedure for an LFP controller that ensures $\mathcal{L}_2$ lateral string stability with respect to the output $\mathbf{y}_i = e_{lat,i}^{des}$ can be summarized as the following (possibly iterative) steps:

    (a) Design the feedback gains $\mathbf{K_P} = [k_{e_{lat}} \quad k_{\tilde{\theta}}]$ and $\mathbf{K_D} = [k_{\dot{e}_{lat}} \quad k_{\dot{\tilde{\theta}}}]$ that are stablizing, and design the feedforward gain $k_{ff}$ based on desired steady-state behaviors, as discussed in Subsection \ref{sec3.3};

    (b) Select the learning gains such that $\mathbf{K_{LD}} < 0$ and $-2k_{e_{lat}} < \mathbf{K_{LP}} < 0$;

    (c) Compute $|D(j\omega)|^2 - |N(j\omega)|^2$ as defined in Eq.~\eqref{eq: string stable sufficient condition} and verify if the coefficient condition in Eq.~\eqref{eq: string stable coefficient condition} is satisfied.
\end{rem}

\section{Numerical simulations}
\label{sec6}
To validate the proposed theoretical findings, we perform numerical simulation tests using parameters based on real-world data. A platoon of 6 vehicles is considered, with the dynamics of each vehicle described by a bicycle model. The parameters of the bicycle model are identified from a Lincoln MKZ (\citep{liu2020lateral}), as listed in Table~\ref{table: parameters}. The platoon operates on a test track at $v_x=10m/s$ with the goal of tracking a desired path. The test track and corresponding desired path are shown in Fig.~\ref{fig: track}, where the grey area denotes the track, the blue line indicates the desired path, and the blue box marks the location where $l_d=0$. The desired path consists of four lane changes over the course of a full circuit, with absolute turning radii ranging from $7.4m$ to $2.9\times 10^4m$, thereby capturing road characteristics of both highway and urban scenarios.
\begin{table}[ht]
\centering
\footnotesize
\caption{Parameters of the Lincoln MKZ}
\begin{tabular}{cccc}
\hline
\hline
\textbf{Parameter} & \textbf{Symbol} & \textbf{Unit} & \textbf{Value}  \\
\hline
Mass & $m$ & kg & 1896 \\
Moment of Inertia & $I_z$ & kg$\cdot$m$^2$ & 3803  \\
Front Cornering Stiffness & $C_{f}$ & N/rad & 400000  \\
Rear Cornering Stiffness & $C_{r}$ & N/rad & 381900  \\
Distance from C.G. to front axle & $a$ & m & 1.2682   \\
Distance from C.G. to rear axle & $b$ &m & 1.5818   \\
\hline
\hline
\label{table: parameters}
\end{tabular}
\end{table}
\begin{figure}[htb!]
    \centering
    \includegraphics[width = 0.4\linewidth]{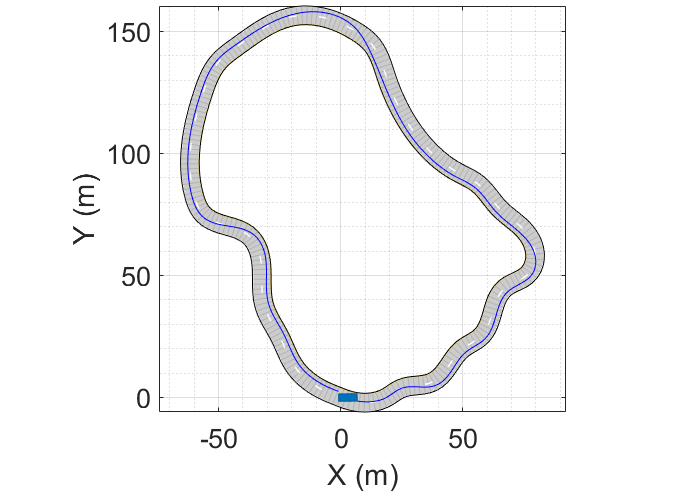}
    \caption{Test track and desired path}
    \label{fig: track}
\end{figure}

Using the setup described in the preceding paragraph, we validate the proposed theoretical findings. Specifically, in Subsection \ref{sec 6.1}, we design the LFP strategy under V2V mode, following the procedure in Remark \ref{rem: design}, and validate that it is $\mathcal{L}_2$ lateral string stable with respect to the output $\mathbf{y}_i=e_{lat,i}^{des}$. In Subsection \ref{sec 6.2}, we set the derivative learning gain $\mathbf{K_{LD}}$ of the LFP strategy to zero, and validate the instability predicted by Proposition \ref{prop: LFP DT K_ld=0 not stable}. Finally, in Subsection \ref{sec 6.3}, we modify the LFP strategy to the FF strategy under onboard sensing mode, and validate the instability results in Propositions \ref{prop: FF PT y=E notstable} and \ref{prop: FF PT y=e notstable}.

\subsection{Learn-from-predecessor strategy under V2V mode}
\label{sec 6.1}
For the LFP strategy under V2V mode, we follow the procedure in Remark~\ref{rem: design} to achieve $\mathcal{L}_2$ lateral string stability with respect to the output $\mathbf{y}_i = e_{lat,i}^{des}$.  
In step~(a) of Remark~\ref{rem: design}, we leverage existing literature by adopting the stabilizing feedback gains $\mathbf{K_P} = [k_{e_{lat}} \quad k_{\tilde{\theta}}]$ and $\mathbf{K_D} = [k_{\dot{e}_{lat}} \quad k_{\dot{\tilde{\theta}}}]$ from~\citep{liu2020lateral}, which were designed for the same Lincoln MKZ vehicle considered in this paper.  
The feedforward gain $k_{ff}$ is adopted from~\citep{rajamani2011vehicle}, which ensures zero steady-state lateral error when tracking a circular arc, and is computed as $k_{ff} = a + b + \frac{m v_x^2}{a+b} \left( \tfrac{b}{C_f} - \tfrac{a}{C_r} + \tfrac{a}{C_r} k_{\tilde{\theta}} \right) - b k_{\tilde{\theta}}$.  
The learning gains $\mathbf{K_{LP}}$ and $\mathbf{K_{LD}}$ (scalars for the output $\mathbf{y}_i = e_{lat,i}^{des}$) are selected within the admissible range specified in step~(b) of Remark~\ref{rem: design}.  
The designed numerical values of all controller gains are listed in Table~\ref{tab: controller gains design}.
\begin{table}[h]
\centering
\footnotesize
\caption{Numerical values of the controller gains for the designed LFP strategy}
\begin{tabular}{c|c|c|c|c|c|c}
\hline \hline
$k_{e_{lat}}$ & $k_{\tilde{\theta}}$ & $k_{\dot e_{lat}}$ & $k_{\dot{\tilde{\theta}}}$ & $k_{ff}$ & $\mathbf{K_{LP}}$ & $\mathbf{K_{LD}}$\\[1.0001ex]
\hline
0.06 & 0.96 & 0  & 0.08 & 1.59 & -0.04  & -0.3 \\
\hline \hline
\end{tabular}
\label{tab: controller gains design}
\end{table}

With the designed gain values, the polynomial $|D(j\omega)|^2 - |N(j\omega)|^2$ in Eq.~\eqref{eq: string stable sufficient condition} has coefficients $a_6 = 2.91\times 10^{22}$, $a_4 = 4.42\times 10^{23}$, $a_2 = 5.70\times 10^{22}$, and $a_0 = 6.07\times 10^{20}$, which satisfy the sufficient condition for $\mathcal{L}_2$ lateral string stability stated in Eq.~\eqref{eq: string stable coefficient condition}.

\begin{figure}[h]
    \centering
    \subfigure[traveled paths in $X-Y$ plane]{\includegraphics[width=0.35\textwidth]{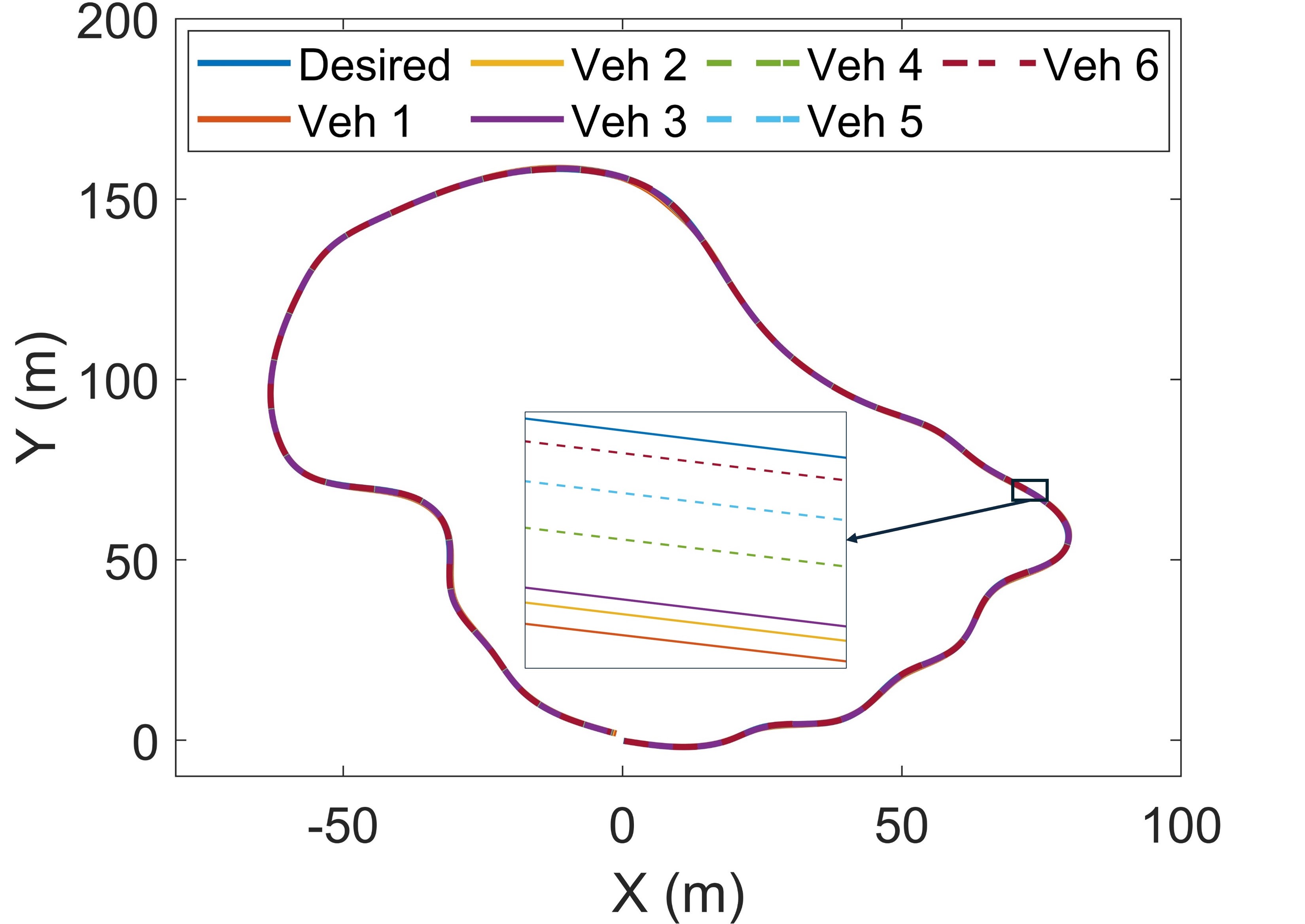}} \\
    \subfigure[lateral error v.s. arc length]{\includegraphics[width=0.35\textwidth]{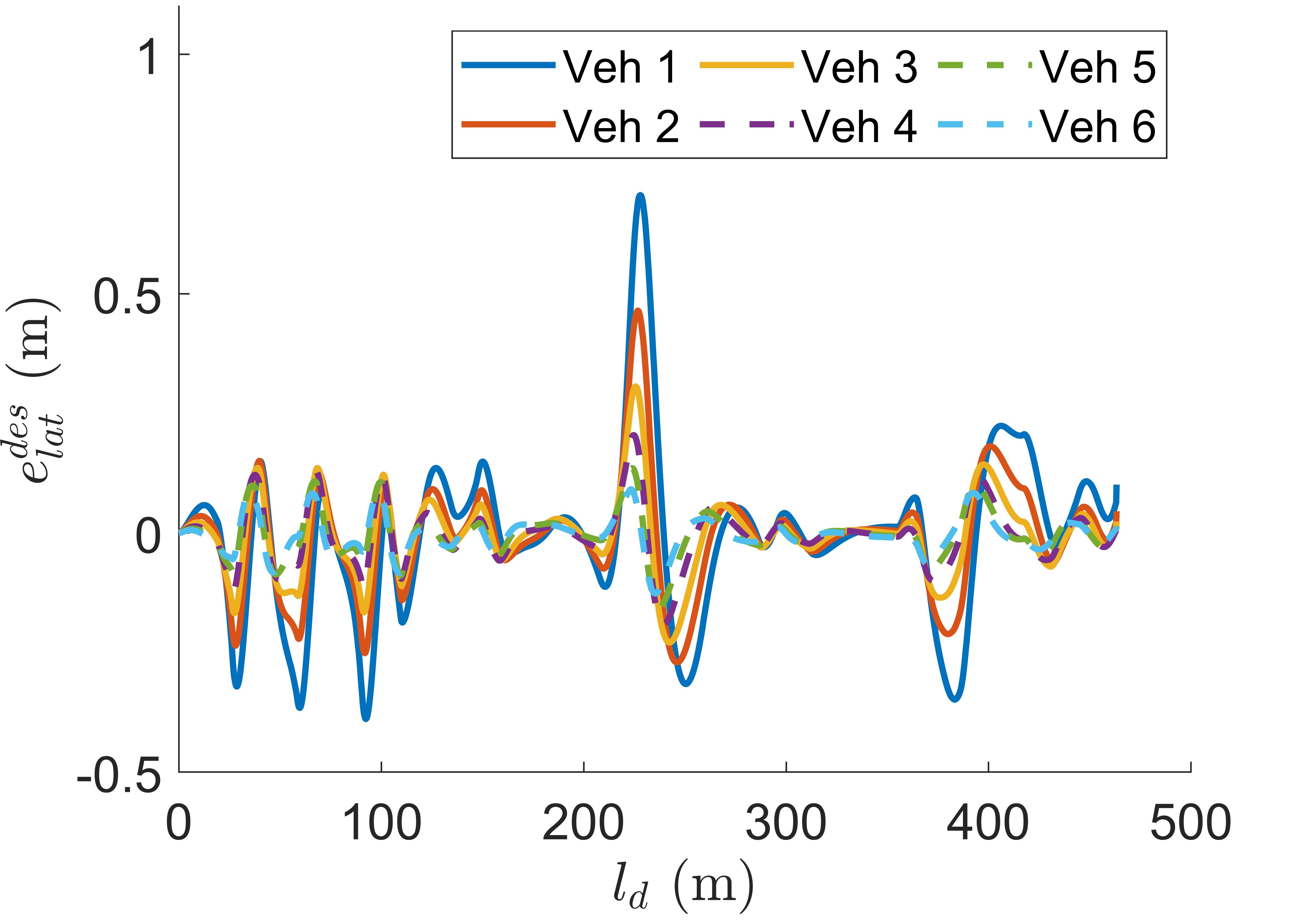}}
    \subfigure[heading error v.s. arc length]
{\includegraphics[width=0.35\textwidth]{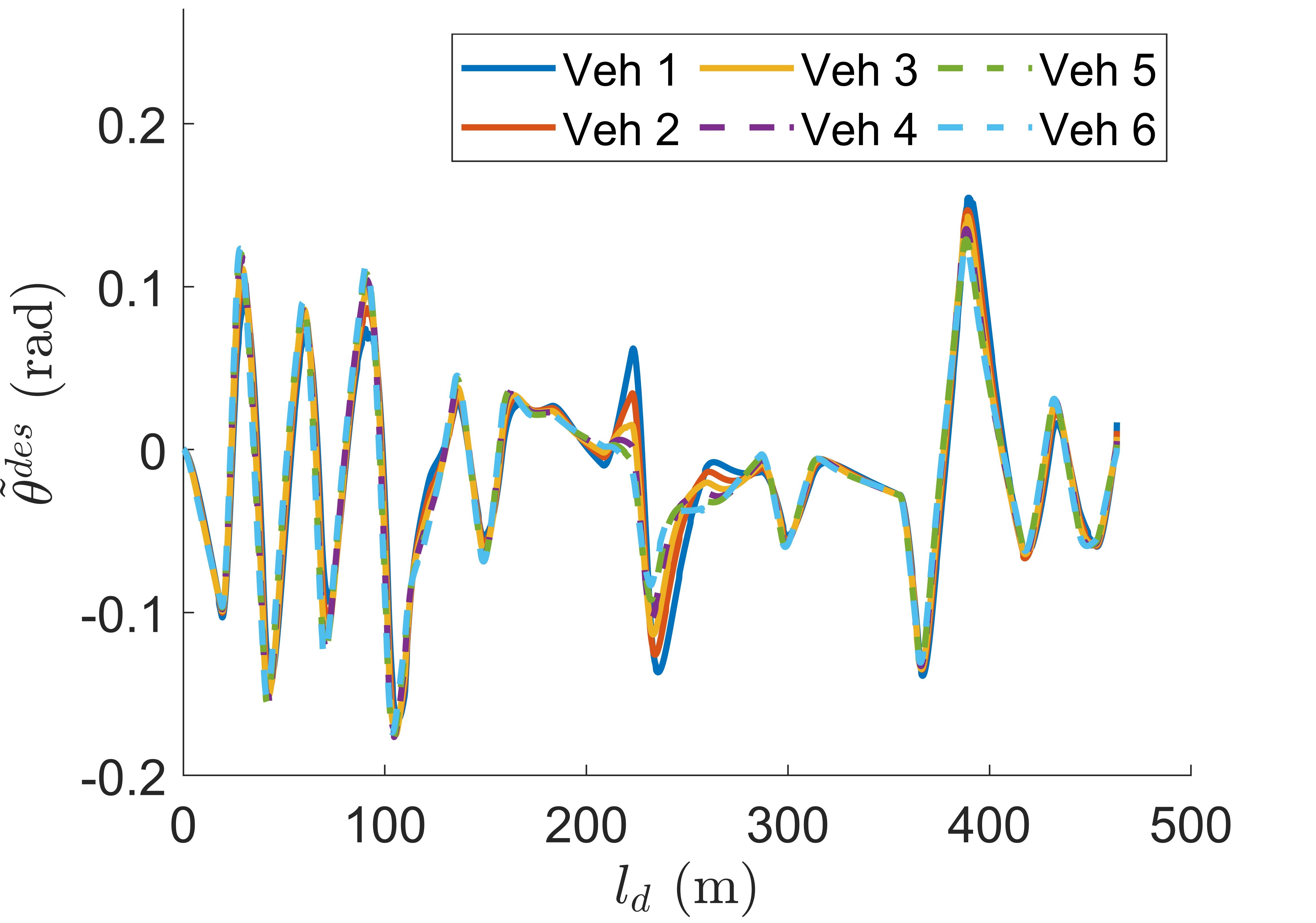}}
    \caption{Simulation results for the LFP strategy under V2V mode}
    \label{fig: LFP DT results}
    \end{figure}
We conduct a simulation experiment using the setup described at the beginning of Section \ref{sec6}. Fig. \ref{fig: LFP DT results} presents (a) the traveled paths in the \(X\)–\(Y\) plane, (b) the lateral error \(e_{lat}^{des}\) versus arc length \(l_d\), and (c) the heading error \(\tilde{\theta}^{des}\) versus \(l_d\) for all vehicles. As seen in Fig. \ref{fig: LFP DT results}(b) and in the zoomed-in view of Fig. \ref{fig: LFP DT results}(a), the lateral error \(e_{lat}^{des}\) consistently diminishes from one vehicle to the next, providing clear visual confirmation of lateral string stability with respect to \(\mathbf{y}_i = e_{lat,i}^{des}\). In contrast, Fig. \ref{fig: LFP DT results}(c) reveals no consistent attenuation pattern in the heading error \(\tilde{\theta}^{des}\).

A more quantitative comparison is given in Fig.~\ref{fig: LFP DT quantitative results}, which plots the $\mathcal{L}_2$ norms of the two outputs of interest: $e_{lat}^{des}$ in Fig.~\ref{fig: LFP DT quantitative results}(a) and $\mathbf{e}^{des}$ in Fig.~\ref{fig: LFP DT quantitative results}(b). The $\mathcal{L}_2$ norms are obtained by approximating the continuous-time integrals with discrete-time summations over the simulation horizon. To better highlight the pattern, the platoon length is extended to 12 vehicles for the quantitative results. As seen in Fig.~\ref{fig: LFP DT quantitative results}(a), the lateral error norm decreases from one vehicle to the next, confirming the $\mathcal{L}_2$ lateral string stability of the LFP strategy with respect to $\mathbf{y}_i = e_{lat,i}^{des}$. Notably, Fig.~\ref{fig: LFP DT quantitative results}(b) indicates that, even though Proposition~\ref{prop: LFP DT y=E not stable} states that $\mathcal{L}_2$ lateral string stability with respect to $\mathbf{y}_i = \mathbf{e}_i^{des}$ is not guaranteed in all cases, it is nevertheless observed to hold empirically in this specific test scenario.
\begin{figure}[h]
    \centering
    \subfigure[$\mathcal{L}_2$ norm of $e_{lat}^{des}$]{\includegraphics[width=0.35\textwidth]{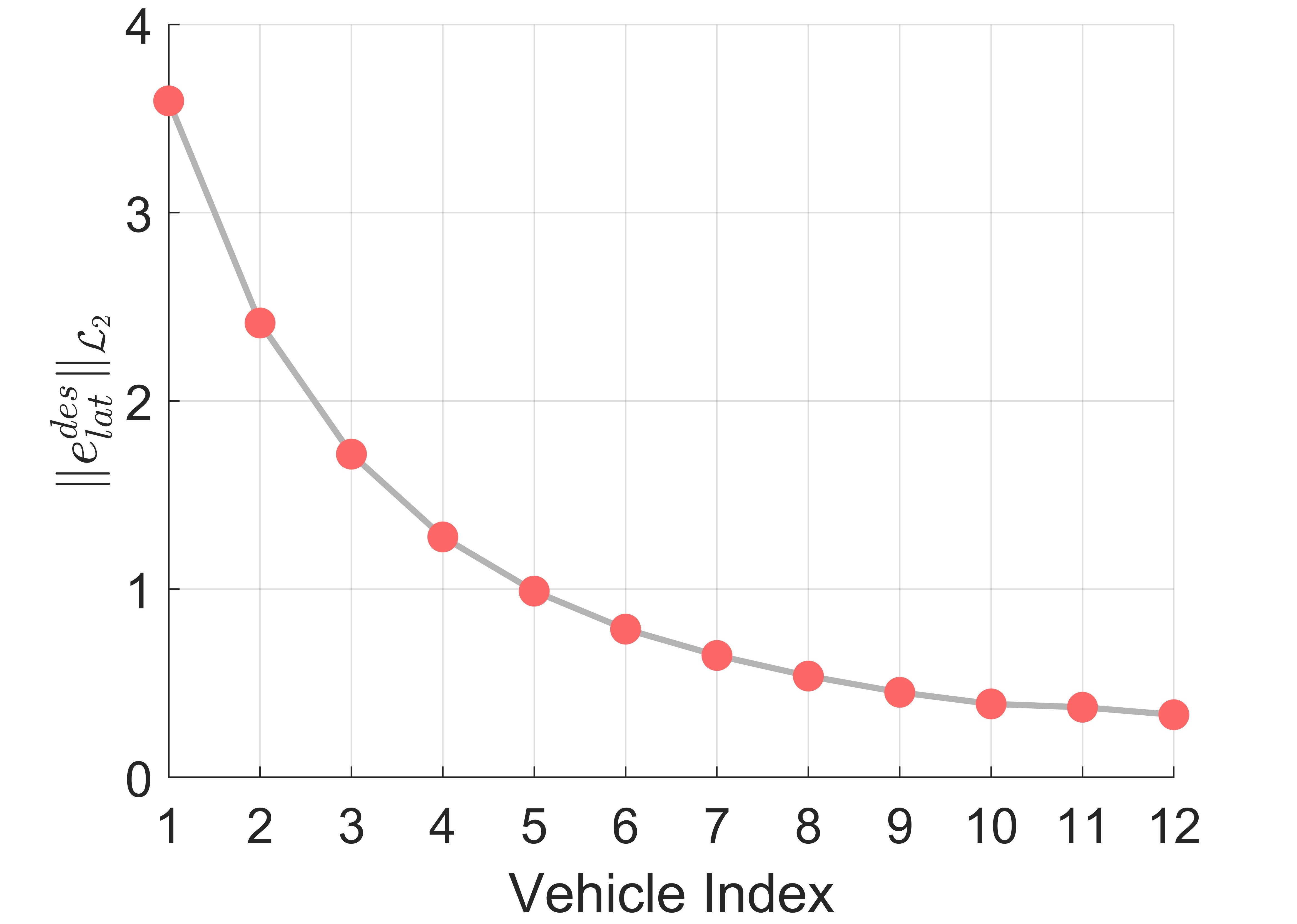}} 
    \subfigure[$\mathcal{L}_2$ norm of $\mathbf{e}^{des}$]{\includegraphics[width=0.35\textwidth]{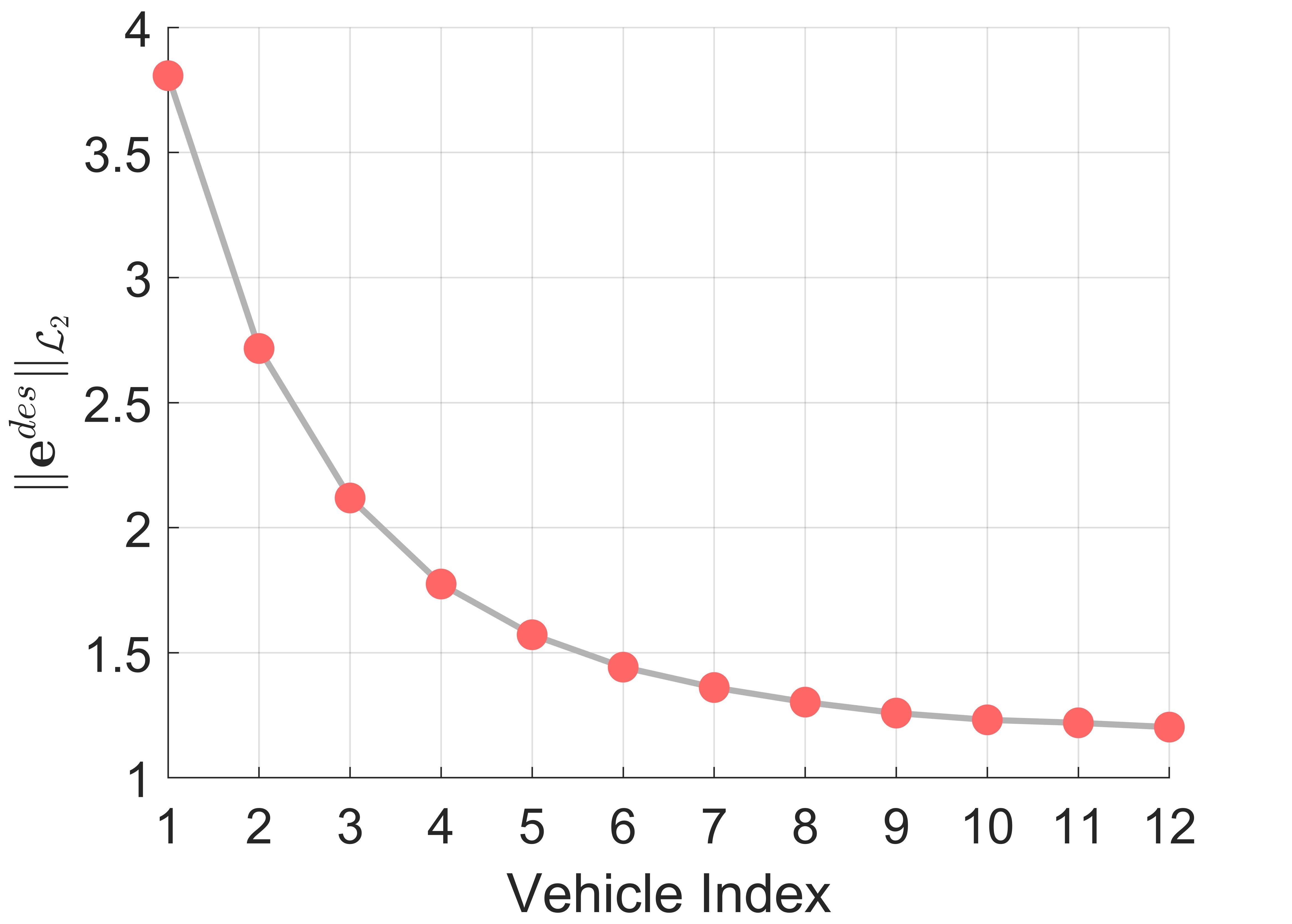}}
    \caption{Quantitative results for the LFP strategy under V2V mode}
    \label{fig: LFP DT quantitative results}
    \end{figure}

\subsection{Learn-from-predecessor strategy under V2V mode with $\mathbf{K_{LD}}=0$}
\label{sec 6.2}
In this subsection, we examine the importance of the derivative learning gain $\mathbf{K_{LD}}$ for $\mathcal{L}_2$ lateral string stability, as established in Proposition \ref{prop: LFP DT K_ld=0 not stable}. Starting from the LFP strategy gains in Subsection~\ref{sec 6.1}, we set $\mathbf{K_{LD}}=0$ while keeping all other gains identical to those in Table~\ref{tab: controller gains design}.

\begin{figure}[h]
    \centering
    \subfigure[traveled paths in $X-Y$ plane]{\includegraphics[width=0.35\textwidth]{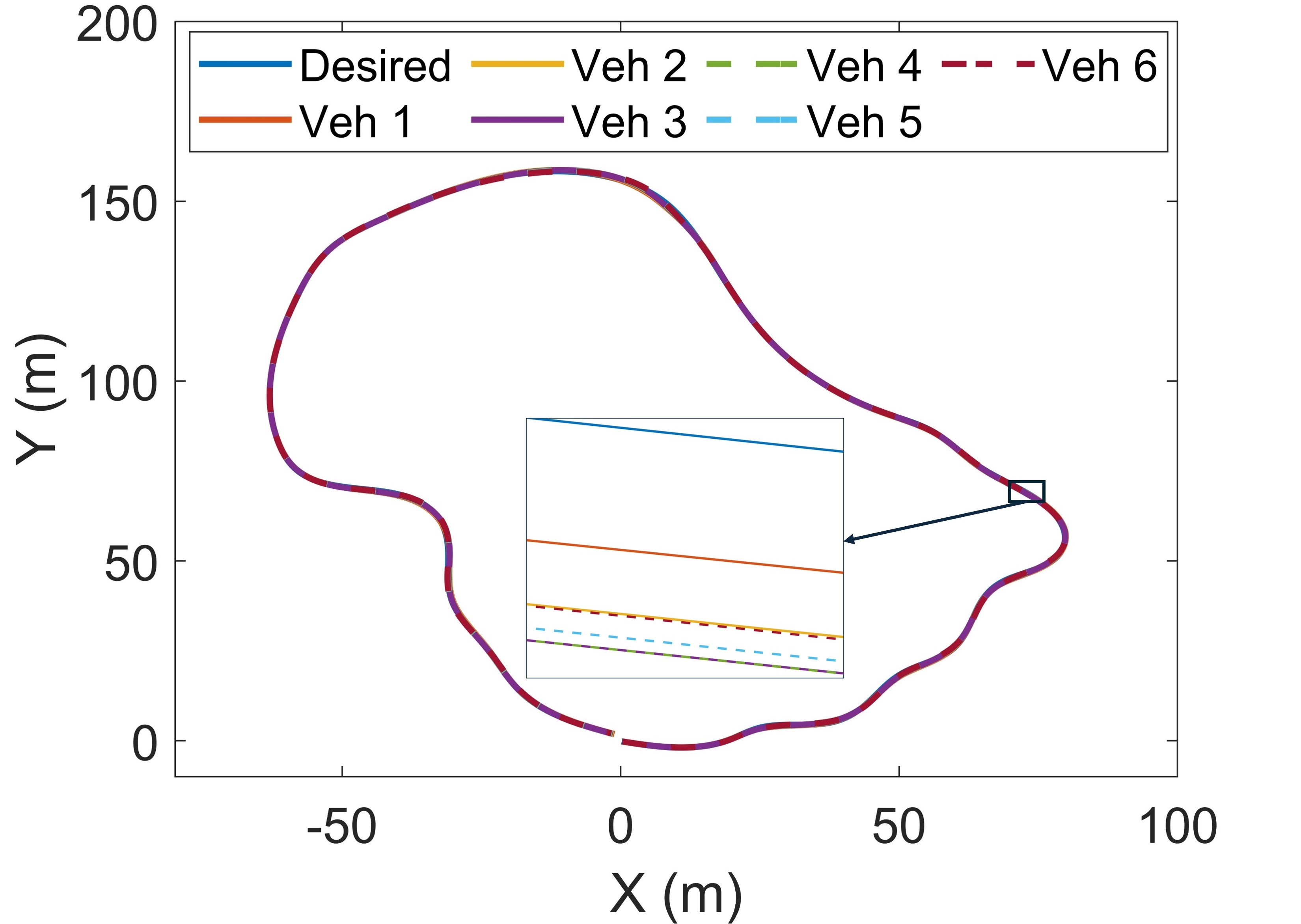}} \\
    \subfigure[lateral error v.s. arc length]{\includegraphics[width=0.35\textwidth]{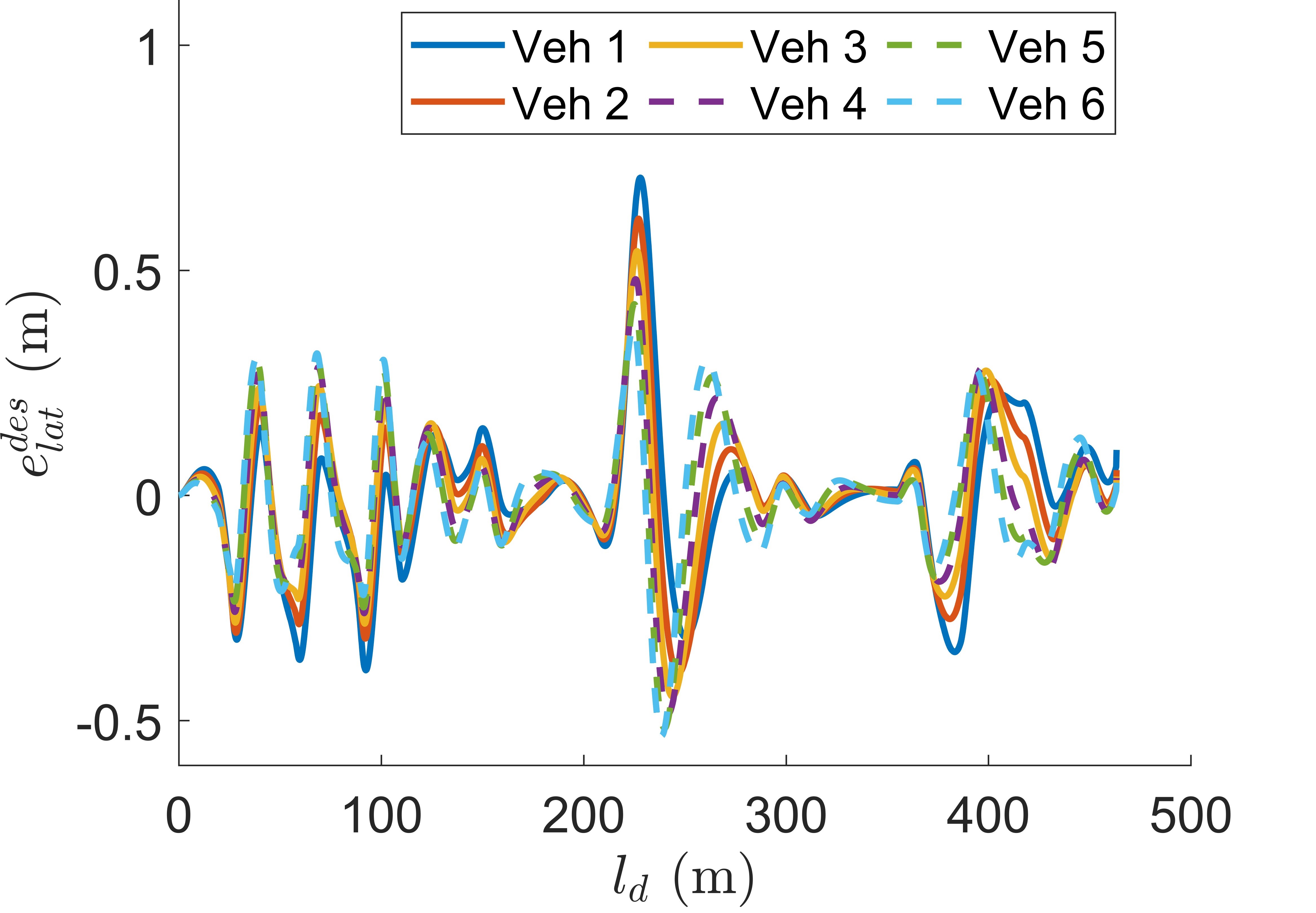}}
    \subfigure[heading error v.s. arc length]
{\includegraphics[width=0.35\textwidth]{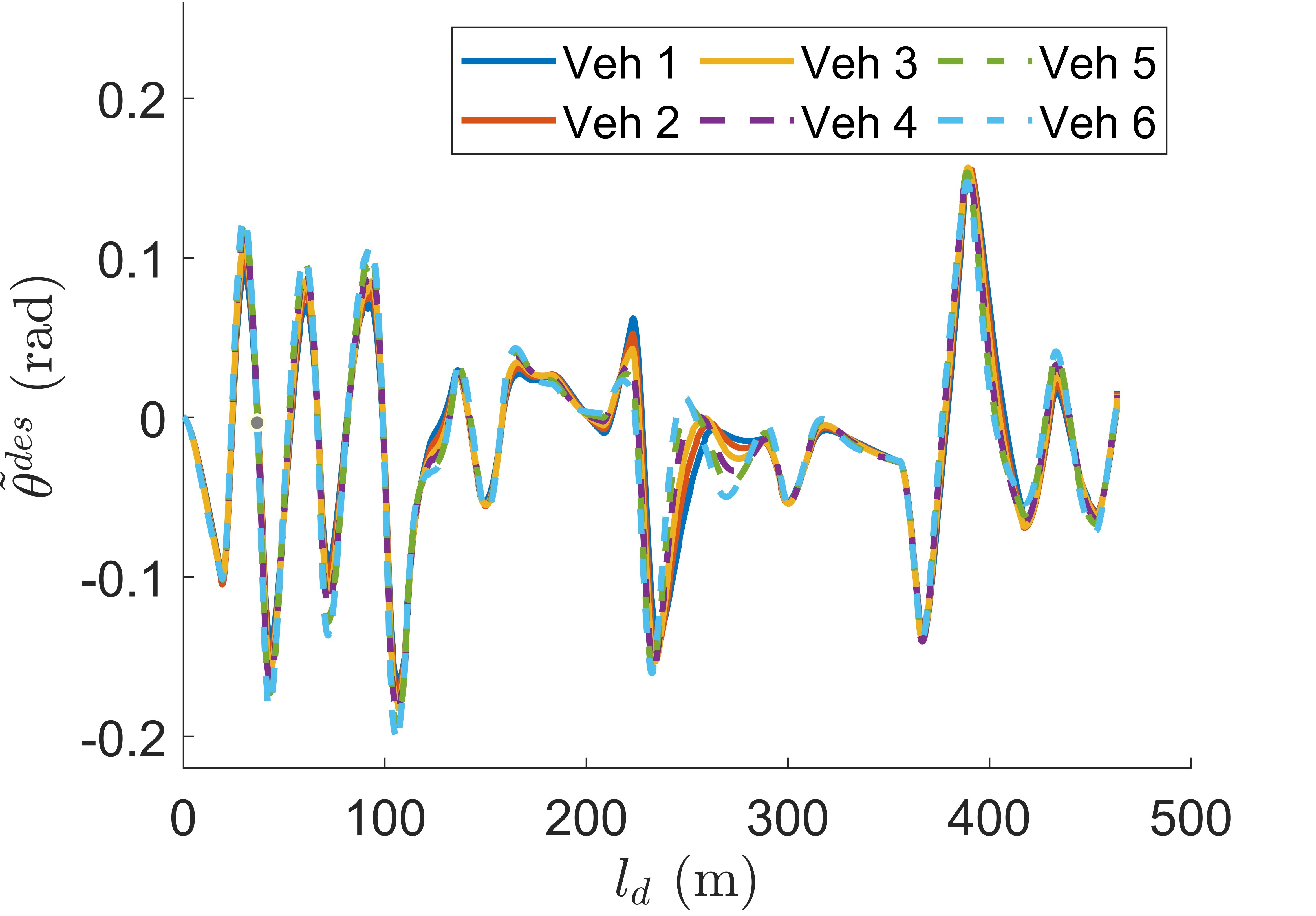}}
    \caption{Simulation results for the LFP strategy under V2V mode with $\mathbf{K_{LD}}=0$}
    \label{fig: LFP DT K_LD results}
    \end{figure}
Using the same simulation setup described at the beginning of Section~\ref{sec6}, Fig.~\ref{fig: LFP DT K_LD results} shows (a) the traveled paths in the $X$–$Y$ plane, (b) the lateral error $e_{lat}^{des}$ versus arc length $l_d$, and (c) the heading error $\tilde{\theta}^{des}$ versus $l_d$ for all vehicles. Compared to Fig.~\ref{fig: LFP DT results}, the error attenuation is notably weaker with $\mathbf{K_{LD}}=0$, and Fig.~\ref{fig: LFP DT K_LD results}(b) even reveals clear amplifications of the lateral error at local peaks.

To highlight this effect more clearly, we extend the platoon size to 12 vehicles and present quantitative results in Fig.~\ref{fig: LFP DT K_LD quantitative results}, which plots the $\mathcal{L}_2$ norms of $e_{lat}^{des}$ and $\mathbf{e}^{des}$ in panels (a) and (b), respectively. The $\mathcal{L}_2$ norms are obtained by approximating the continuous-time integrals with discrete-time summations over the simulation horizon. As shown, the lateral error norm begins to grow beyond the 5$^{\text{th}}$ vehicle, and the error vector norm beyond the 4$^{\text{th}}$, confirming the lack of $\mathcal{L}_2$ lateral string stability for both output definitions $\mathbf{y}_i = e_{lat,i}^{des}$ and $\mathbf{y}_i = \mathbf{e}_i^{des}$. These findings align with the instability predicted in Propositions~\ref{prop: LFP DT K_ld=0 not stable} and~\ref{prop: LFP DT y=E not stable}.
\begin{figure}[h]
    \centering
    \subfigure[$\mathcal{L}_2$ norm of $e_{lat}^{des}$]{\includegraphics[width=0.35\textwidth]{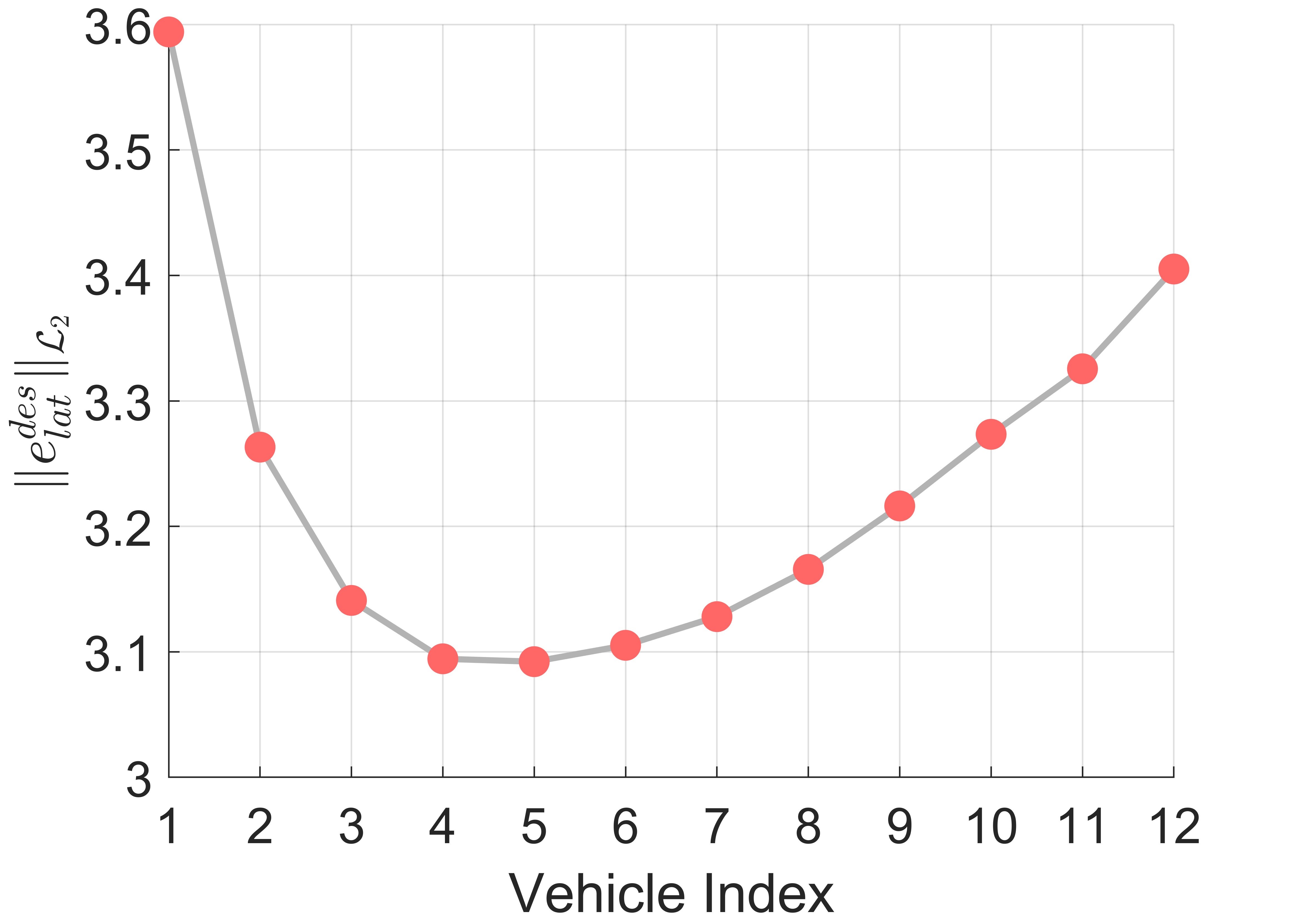}} 
    \subfigure[$\mathcal{L}_2$ norm of $\mathbf{e}^{des}$]{\includegraphics[width=0.35\textwidth]{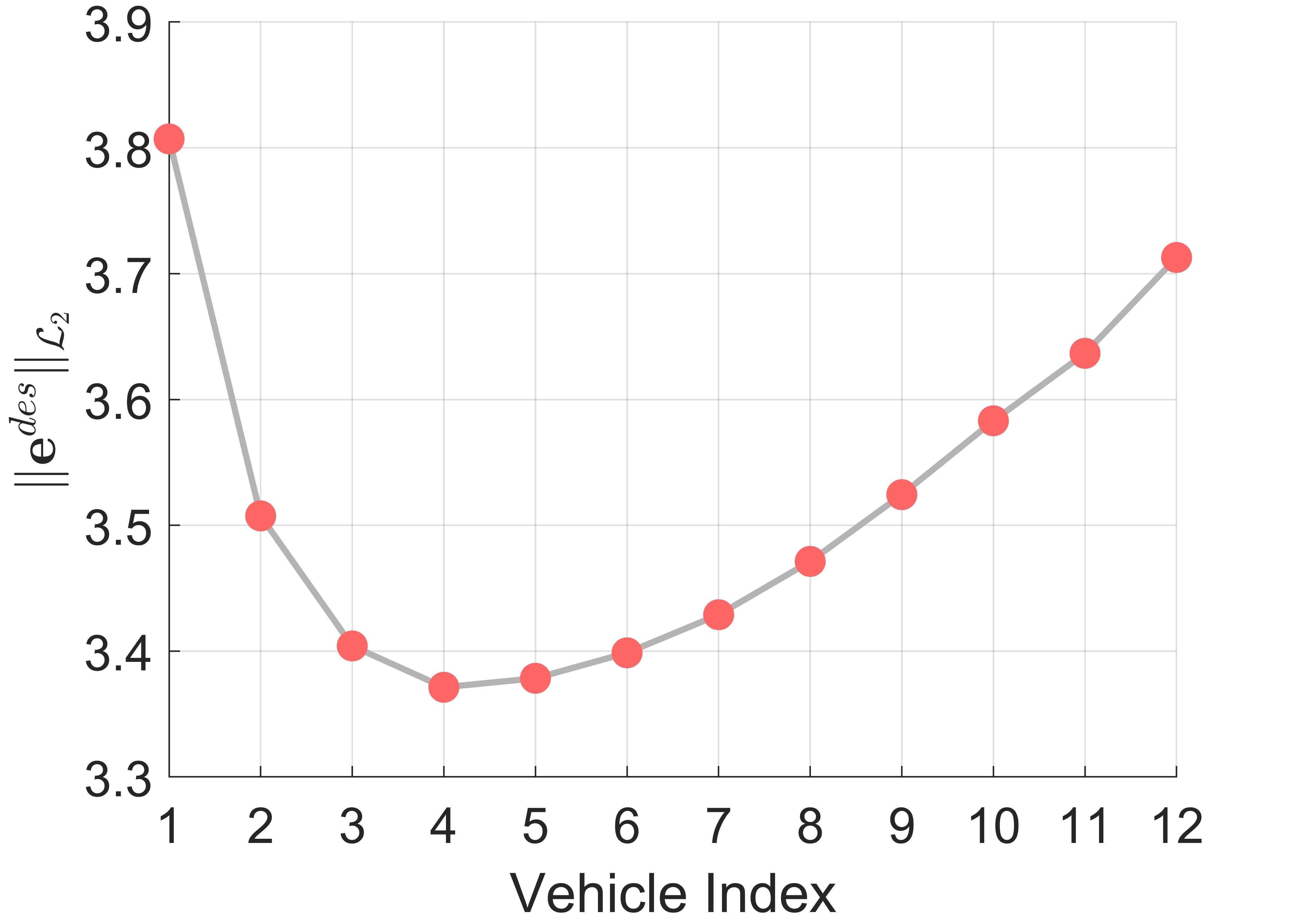}}
    \caption{Quantitative results for the LFP strategy under V2V mode with $\mathbf{K_{LD}}=0$}
    \label{fig: LFP DT K_LD quantitative results}
    \end{figure}

\subsection{Feedback-feedforward strategy under onboard sensing mode}
\label{sec 6.3}
In this subsection, we validate the instability results of the FF strategy under onboard sensing mode, as established in Propositions \ref{prop: FF PT y=E notstable} and \ref{prop: FF PT y=e notstable}. The control strategy switches from LFP to FF, while the values of the common controller gains $k_{e_{lat}}, k_{\tilde{\theta}}, k_{\dot e_{lat}}, k_{\dot{\tilde{\theta}}}$, and $k_{ff}$ are kept identical to those in Table~\ref{tab: controller gains design}. The tracking scheme is changed from DT in V2V mode to PT in onboard sensing mode, meaning that the tracking errors used for control input computation of each following vehicle are calculated relative to the recorded path traveled by its immediate predecessor.

Using the same simulation setup described at the beginning of Section~\ref{sec6}, Fig.~\ref{fig: FF results} shows (a) the traveled paths in the $X$–$Y$ plane, (b) the lateral error $e_{lat}^{des}$ versus arc length $l_d$, and (c) the heading error $\tilde{\theta}^{des}$ versus $l_d$ for all vehicles. Across all subfigures, a pronounced amplification of errors along the platoon is evident, consistent with the theoretical findings in Propositions~\ref{prop: FF PT y=E notstable} and~\ref{prop: FF PT y=e notstable}.
\begin{figure}[h]
    \centering
    \subfigure[traveled paths in $X-Y$ plane]{\includegraphics[width=0.35\textwidth]{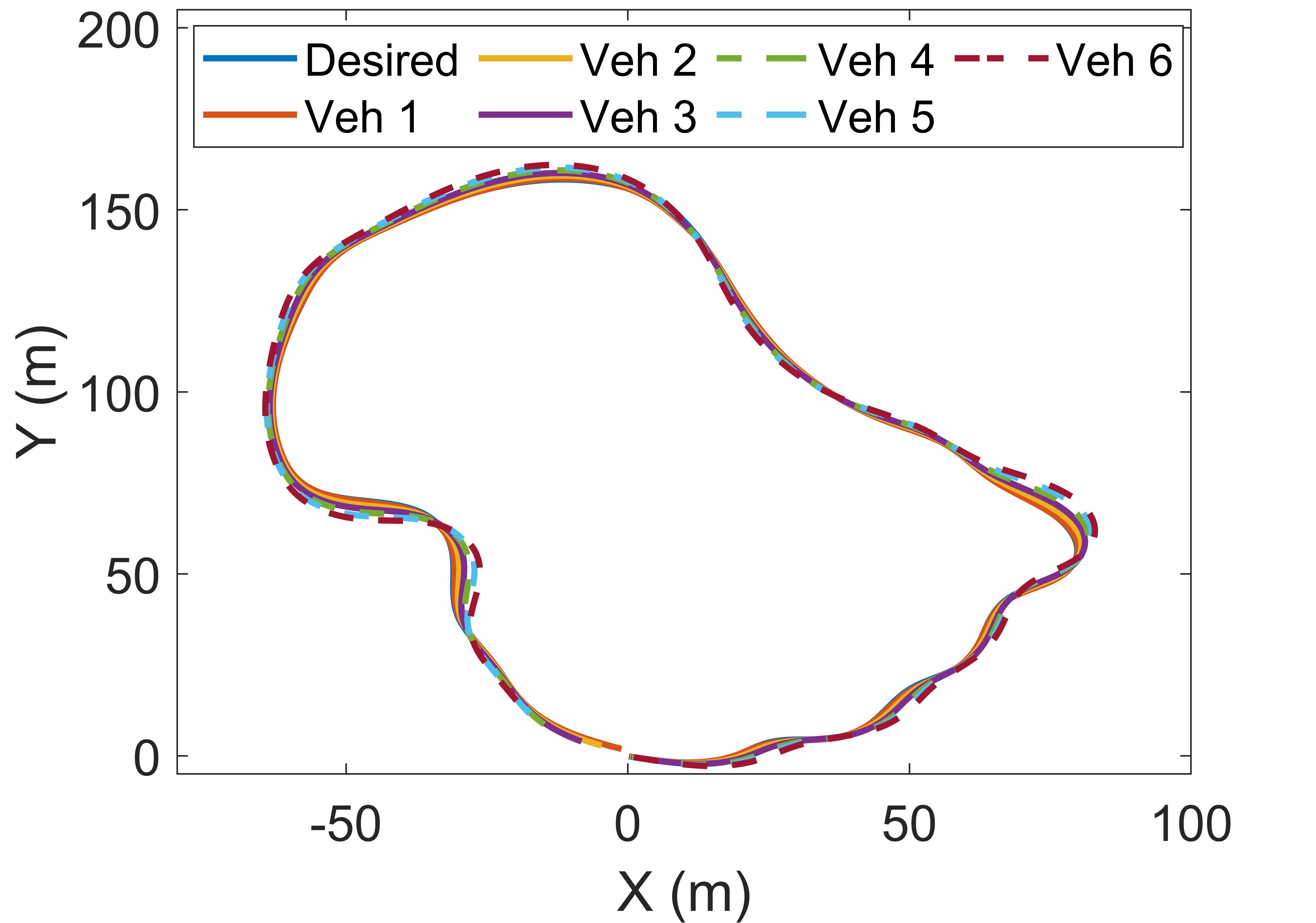}} \\
    \subfigure[lateral error v.s. arc length]{\includegraphics[width=0.35\textwidth]{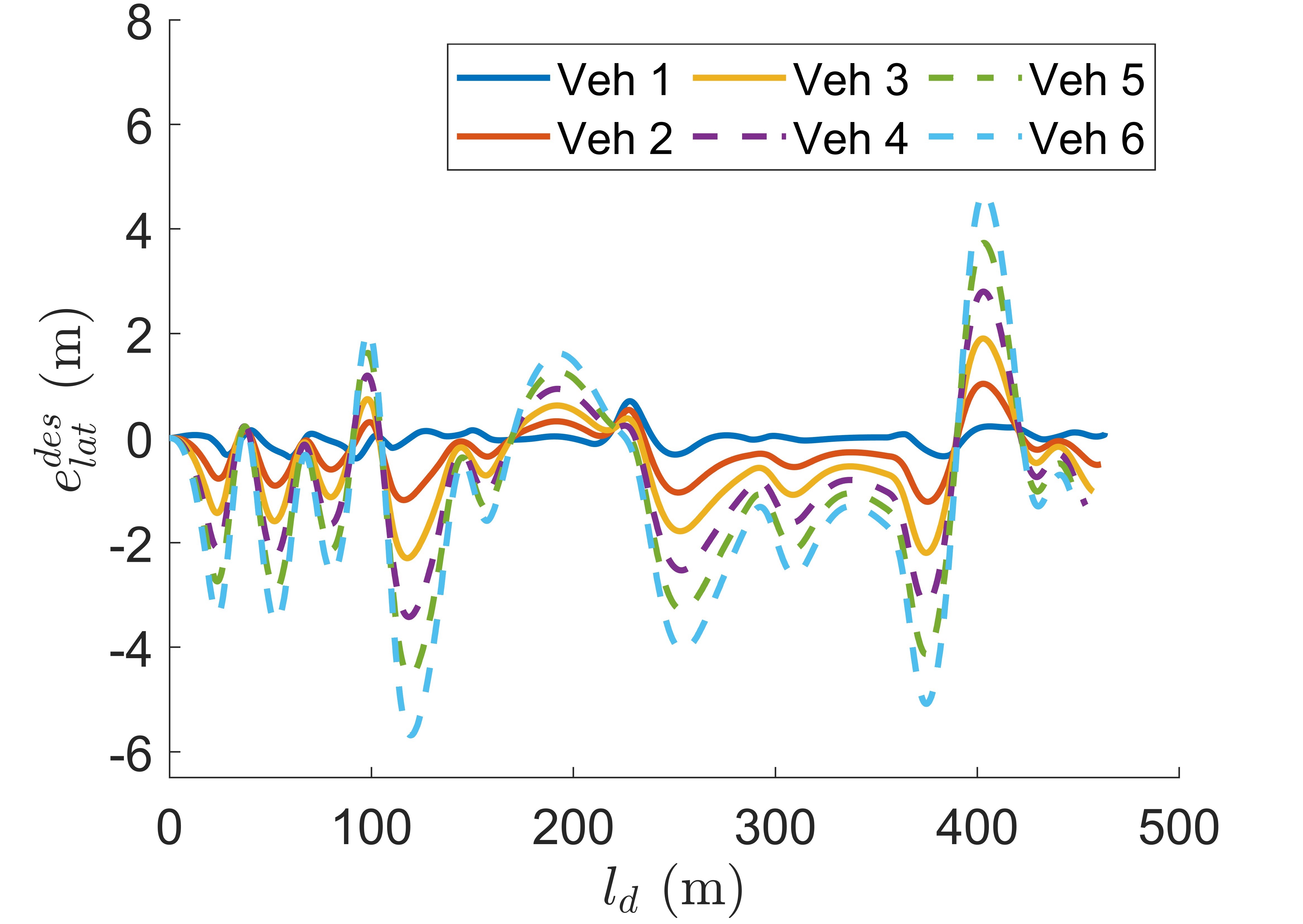}}
    \subfigure[heading error v.s. arc length]
{\includegraphics[width=0.35\textwidth]{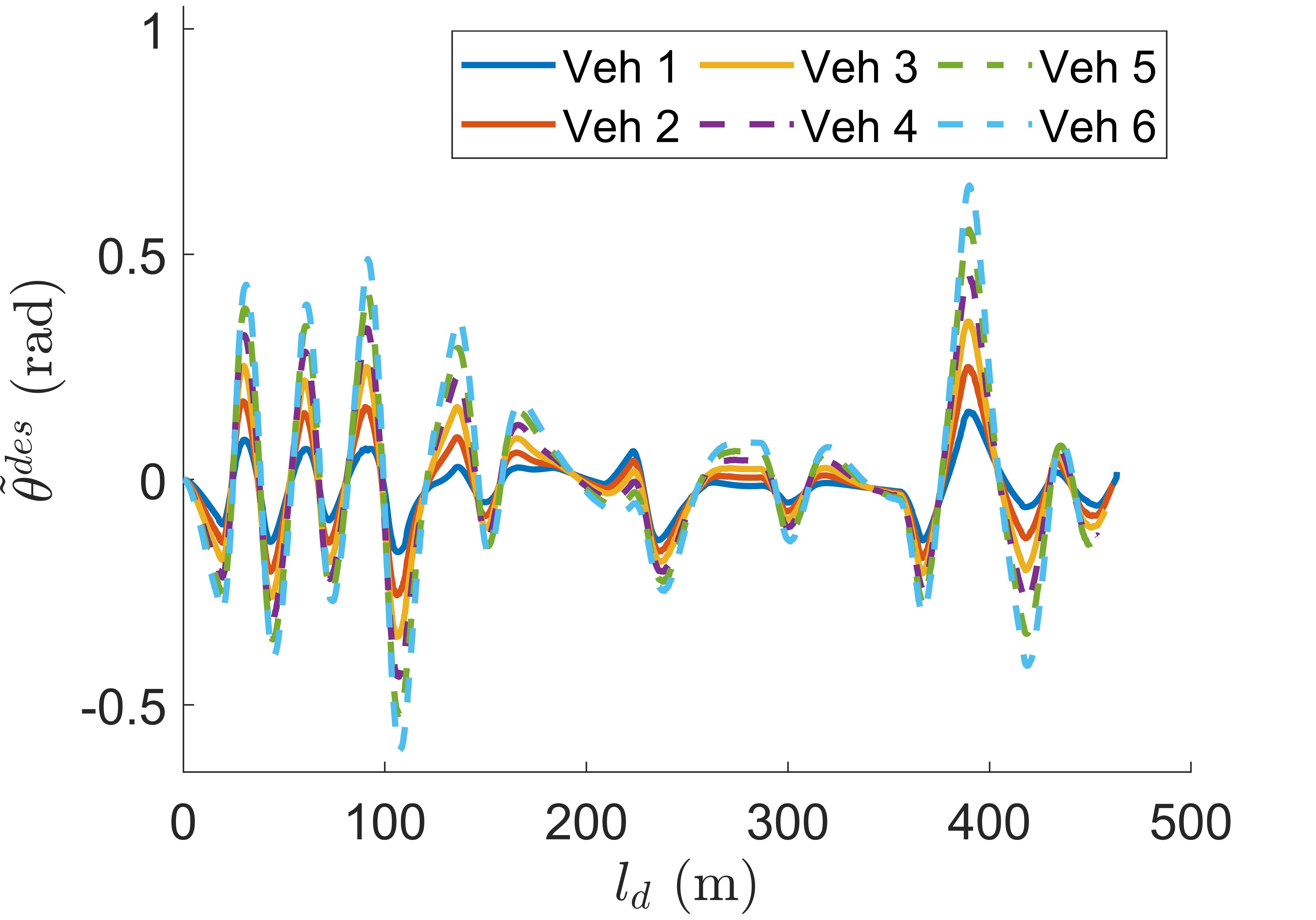}}
    \caption{Simulation results for the FF strategy under onboard sensing mode}
    \label{fig: FF results}
    \end{figure}

Similarly, we extend the platoon size to 12 vehicles and present quantitative results in Fig.~\ref{fig: FF quantitative results}, which plots the $\mathcal{L}_2$ norms of $e_{lat}^{des}$ and $\mathbf{e}^{des}$ in panels (a) and (b), respectively. Both the lateral error norm and the error vector norm exhibit a strict increase from one vehicle to the next, once again confirming the instability results in Propositions~\ref{prop: FF PT y=E notstable} and~\ref{prop: FF PT y=e notstable}.
\begin{figure}[h]
    \centering
    \subfigure[$\mathcal{L}_2$ norm of $e_{lat}^{des}$]{\includegraphics[width=0.35\textwidth]{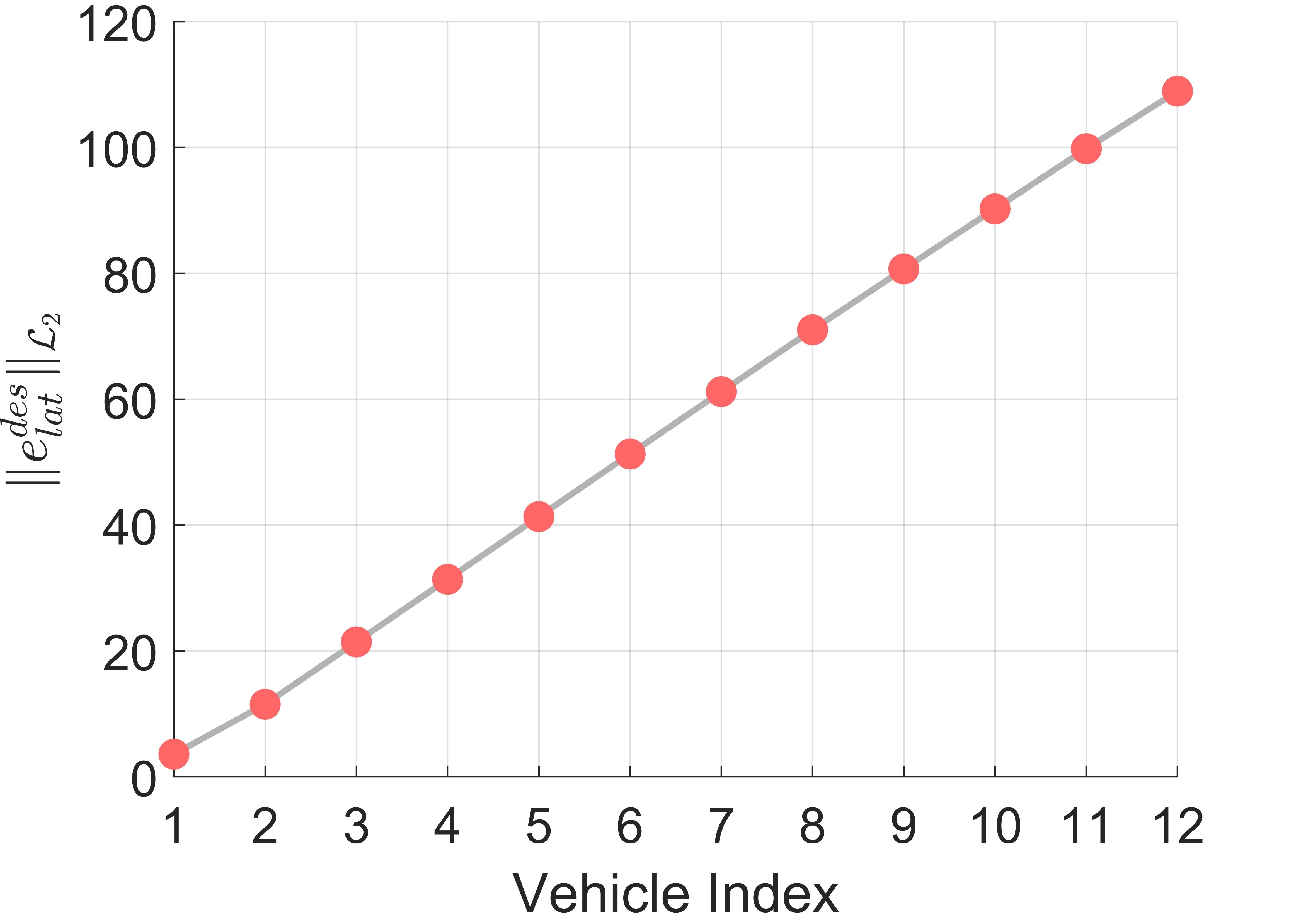}} 
    \subfigure[$\mathcal{L}_2$ norm of $\mathbf{e}^{des}$]{\includegraphics[width=0.35\textwidth]{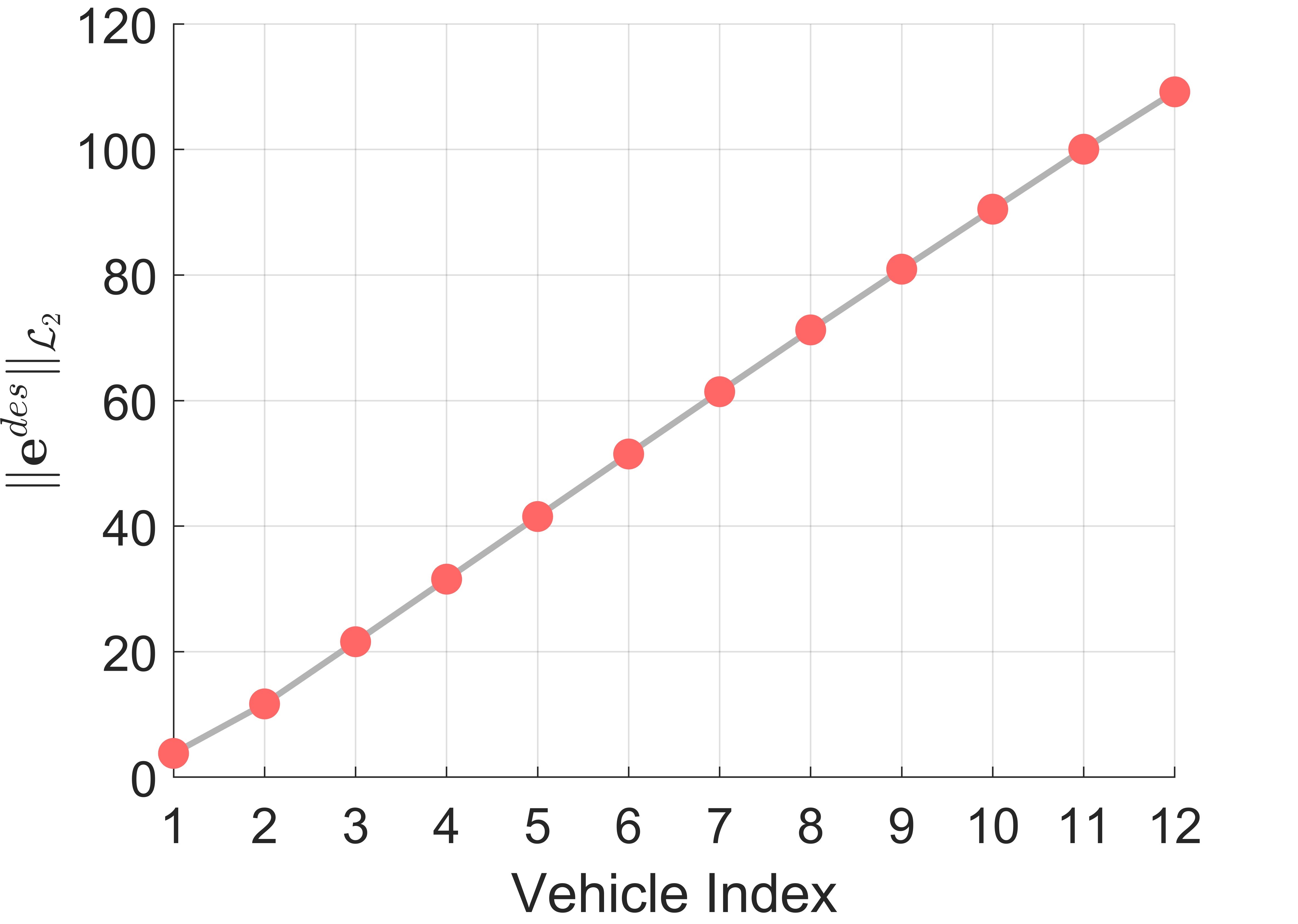}}
    \caption{Quantitative results for the FF strategy under onboard sensing mode}
    \label{fig: FF quantitative results}
    \end{figure}

\section{Conclusion}
\label{sec7}
In summary, this work provided a systematic formulation of the lateral control problem for vehicle platoons using arc-length reparameterization and proposed a formal definition of $\mathcal{L}_2$ lateral string stability. A necessary and sufficient condition of lateral string stability was derived. Within this framework, we examined lateral string stability for two control strategies: a FF strategy using only onboard sensors and an LFP strategy utilizing V2V communication. Both strategies are evaluated with respect to two output selections: the full tracking error vector and the scalar lateral error. Detailed analyses were carried out.

Our analysis revealed that only one of the four cases can achieve $\mathcal{L}_2$ lateral string stability. In particular, the LFP strategy with output $\mathbf{y_i}=e_i^{lat}$ is the sole case that ensures error attenuation rather than amplification along the platoon. These results have important implications for lateral controller design in vehicle platoons, where the onboard sensors of the following vehicles can be occluded. Our findings underscore the critical role of incorporating desired path information and predecessor information into the design. Moreover, our study emphasizes the importance of the derivative learning term in the LFP strategy, with clear guidelines on controller design provided. The findings were validated through comprehensive numerical simulations.


\appendix
\section{Proof of Proposition \ref{prop: FF PT y=e notstable}}
\label{apped A}
From Eq. \eqref{eq: FF PT SISO transfer}, it is evident that the output $e_{lat,2}^{des}$ depends on both inputs $e_{lat,1}^{des}$ and $\tilde{\theta}_1^{des}$, making Theorem \ref{prop: LFP DT frequency-domain condition} inapplicable. However, we can prove that the platoon is not $\mathcal{L}_2$ lateral string stable by finding a counterexample. Consider the case where the desired path is a circular arc, so that $\left(\theta^{des}(l_d)\right)'\equiv c_1$ where $c_1$ is a nonzero constant. Let $e_{lat,1}^{des}(l_d)\equiv0$, meaning that the first vehicle tracks the desired path with no lateral error. In this case, the first vehicle has a non-zero steady-state heading error (see, e.g., Section 3.2 of \citep{rajamani2011vehicle}), that is, $\lim_{l_d\rightarrow\infty}\tilde{\theta}_1^{des}(l_d)=c_2$ where $c_2$ is a non-zero constant. Hence, by the final value theorem, 
\begin{equation}
\label{eq: FF PT theta final value}
\lim_{s\rightarrow0}s\tilde{\theta}_1^{des}(s)=\lim_{l_d\rightarrow\infty}\tilde{\theta}_1^{des}(l_d)=c_2.
\end{equation}
Furthermore, since $e_{lat,1}^{des}(l_d)\equiv0$, we have $e_{lat,1}^{des}(s)=0$. Hence, Eq. \eqref{eq: FF PT SISO transfer} becomes
\begin{equation}
    e_{lat,2}^{des}(s)=H_2(s)\tilde{\theta}_1^{des}(s)
\end{equation}

Here, we recall that $\mathbf{\hat{M}}(s)=s^2v_x^2\mathbf{M} + sv_x\mathbf{C}  + \mathbf{L}$ and $\mathbf{K_{fb}}(s)=\mathbf{K_P} + s v_x \mathbf{K_D}$, defined in Eqs. \eqref{model Laplace} and \eqref{eq: u Laplace FF PT}, respectively. We proceed to evaluate $H_2(s)$ at $s=0$, note that the first equality below holds since there is no pole-zero cancellation at $s=0$:
\begin{align}
H_2(0)&={[1~~0]\left({\mathbf{\hat{M}}(0)}+\mathbf{B}\mathbf{K_{fb}}(0)\right)^{-1}\mathbf{B}\mathbf{K_{fb}}(0)}\begin{bmatrix}
    0 \\ 1
\end{bmatrix}  \notag \\
&= [1~~0] \left(\mathbf{L} + \mathbf{B} \mathbf{K_{P}} \right)^{-1} \mathbf{B} \mathbf{K_{P}}\begin{bmatrix}
    0 \\ 1
\end{bmatrix} \notag \\
&= [1~~0] \begin{bmatrix}
    \tfrac{a  C_f  k_{\tilde{\theta}} - a  C_f + b  C_r}{a  C_f  C_r  k_{e_{lat}} + b  C_f  C_r  k_{e_{lat}}} & \tfrac{ -C_f  k_{\tilde{\theta}}+ C_f+ C_r}{a  C_f  C_r  k_{e_{lat}} + b  C_f  C_r  k_{e_{lat}}} \\[1.5ex]
    -\tfrac{a}{a  C_r+b  C_r} & \tfrac{1}{a  C_r+b C_r} 
\end{bmatrix}
\begin{bmatrix}
        C_f k_{e_{lat}} & C_f k_{\tilde{\theta}} \\[1.5ex]
        a C_f k{e_{lat}} & a C_f k_{\tilde{\theta}}
    \end{bmatrix}
    \begin{bmatrix}
    0 \\ 1
\end{bmatrix}  \notag \\[1.5ex]
&= [1~~0] \begin{bmatrix}
    1 & \tfrac{k_{\tilde{\theta}}}{k_{e_{lat}}} \\[2ex]
    0 & 0
    \end{bmatrix}
    \begin{bmatrix}
    0 \\ 1
\end{bmatrix} \notag \\
&=\frac{k_{\tilde{\theta}}}{k_{e_{lat}}}.
\end{align}
Therefore, by the final value theorem and using Eq. \eqref{eq: FF PT theta final value},
\begin{equation}
    \lim_{l_d\rightarrow\infty}e_{lat,2}^{des}(l_d)=\lim_{s\rightarrow0}se_{lat,2}^{des}(s)=\lim_{s\rightarrow0}sH_2(s)\tilde{\theta}_1^{des}(s)=H_2(0)\lim_{s\rightarrow0}s\tilde{\theta}_1^{des}(s)=\frac{c_2k_{\tilde{\theta}}}{k_{e_{lat}}}\neq0.
\end{equation}
Hence, $e_{lat,2}^{des}(l_d)\nequiv0$ for this counterexample, then in the arc length domain, with zero initial states, we have
\begin{equation}
    \left\|e_{lat,2}^{des}(l_d)\right\|_{\mathcal{L}_2}>\left\|e_{lat,1}^{des}(l_d)\right\|_{\mathcal{L}_2}=0,
\end{equation}
and the condition in Eq. \eqref{eq:strict_L2_string_stability} in Definition \ref{def:l2_lateral_string_stability} is violated.

\printcredits

\bibliographystyle{cas-model2-names}
\bibliography{root}

\end{document}